\documentclass[11pt,twoside]{article}

\usepackage{geometry}
\input{commands-full}

\begin{document}

\begin{center}

  {\LARGE
  \textbf{Leave a Window Out: Modifying the Jackknife for Predictive Inference in Time Series}
  
  }

\vspace*{.2in}

{\large{
\begin{tabular}{ccc}
Hanyang Jiang$^\star$, Rina Foygel Barber$^\ddagger$, Ashwin Pananjady$^{\star, \dagger}$, Yao Xie$^\star$
\end{tabular}
}}
\vspace*{.2in}

\begin{tabular}{c}
Schools of Industrial and Systems Engineering$^\star$ and
Electrical and Computer Engineering$^\dagger$, \\
Georgia Tech \\
Department of Statistics, University of Chicago$^\ddagger$
\end{tabular}

\vspace*{.2in}

\today

\vspace*{.2in}

\begin{abstract}
Conformal prediction methods enjoy strong theoretical and empirical predictive inference performance, provided the data is exchangeable and is treated symmetrically during training. However, these assumptions are impractical in many settings, such as time series, where temporal dependence violates exchangeability and it is preferable to use predictors that leverage dependence by treating data asymmetrically. Recent work shows that split conformal prediction is robust to these issues, but sample splitting can reduce accuracy, motivating the study of methods that do not rely on data splitting in the time series setting.

In this work, we show that the vanilla leave-one-out jackknife can suffer arbitrary loss of coverage even in canonical time series models with mild temporal dependence. As a remedy, we propose a modification tailored to such settings, which we term the leave-a-window-out (LWO) method, and show that it can achieve valid coverage provided that the model-fitting procedure satisfies mild stability properties. Our proofs are based on quantifying the degree to which the data departs from cyclic exchangeability, which we introduce new coefficients to measure. Experiments on time series demonstrate that our method often enjoys valid coverage when the vanilla jackknife fails to cover, while producing much narrower intervals than split conformal prediction.
\end{abstract}
\end{center}

\section{Introduction} \label{sec:intro}

Quantifying the uncertainty of predictions is a fundamental challenge in data science, including in time-series analysis and sequential decision-making. As machine learning models are increasingly deployed for forecasting in high-stakes domains such as financial markets \citep{sezer2020financial}, climate change \citep{rolnick2022tackling}, and energy applications~\citep{hong2016probabilistic}, it is essential to provide valid prediction intervals around their outputs. 

In general, uncertainty quantification can be performed by making assumptions about how the data is generated. On the one hand, classical time-series models are explicitly designed to accommodate dependent data, but provide valid uncertainty quantification only under strong parametric assumptions on the model. However, these parametric assumptions are often violated when using modern black-box predictors like neural networks.
On the other hand, ``distribution-free” predictive inference provides a suite of wrapper techniques that provide uncertainty quantification for predictions arising from any underlying model \citep{vovk2005algorithmic,shafer2008tutorial}, and the validity of these methods relies only on \emph{exchangeability} of data. However, dependent data (e.g., in time series and other sequential scenarios) is not exchangeable.

This drawback notwithstanding, distribution-free predictive inference methods still enjoy widespread use across applications in which exchangeability may be violated. In fact, a particular such method called split conformal prediction has recently been shown to exhibit automatic robustness to violations of exchangeability that are often observed in time series~\citep{oliveira2024split, barber2025predictive}. While split conformal is conceptually simple and easy to implement, it involves sample-splitting and can sacrifice statistical efficiency. In contrast, predictive inference methods based on the leave-one-out  principle---colloquially referred to as variations of the \emph{jackknife}~\citep{tukey1958bias,stone1974cross,efron1983leisurely}---are appealing since they avoid sample-splitting and make full use of the dataset. While they can be computationally more intensive than conformal methods based on data splitting, leave-one-out methods offer greater statistical efficiency and are preferable in low-data regimes. However, the behavior of leave-one-out methods in settings with dependent data is not well understood. Are these methods also robust to similar violations of exchangeability---or are modifications of these methods needed in order to obtain reasonable coverage properties in a time series setting? This forms the motivating question of our paper.

\subsection{Using the jackknife for predictive inference} \label{sec:jk-intro}

To make our discussion concrete, let us formally introduce the jackknife method for memoryless predictors. Suppose we have a time series of covariate-response data $(X_1,Y_1),\ldots,(X_{n+1},Y_{n+1})$, with $(X_i,Y_i) \in \mathcal{X} \times \mathcal{Y}$ representing the $i$-th data point. The data point at index $n+1$ is considered to be the ``test point'', with $X_{n+1}$ observed but $Y_{n+1}$ unobserved, while the points indexed $1,\dots,n$ serve as the training data. Our goal is to construct a prediction set $\widehat{C}_n(X_{n+1}) \subseteq \mathcal{Y}$ based on the training data $(X_1,Y_1),\ldots,(X_n,Y_n)$ and $X_{n+1}$ such that this set contains the unknown response $Y_{n+1}$ with a pre-specified probability:
\begin{align} \label{eq:coverage-desired}
\mathbb{P}(Y_{n+1} \in \widehat{C}_n(X_{n+1})) \ge 1 - \alpha.
\end{align}
Here, $\alpha \in (0,1)$ denotes the target miscoverage level.

The \emph{jackknife}, or leave-one-out, method produces a prediction interval by repeatedly fitting a model on $n - 1$ out of the $n$ training data points and constructing conformity scores on the point that is left out. Concretely, suppose our model-fitting algorithm is memoryless (in that it produces a map $\mathcal{X} \to \mathcal{Y}$) and that conformity is measured by the absolute residual. For each\footnote{Throughout, we use the shorthand notation $[k] := \{1, \ldots, k\}$.} $i \in [n]$, the jackknife computes the predictor $\widehat{f}_{-i}$ by training our model-fitting algorithm on all data points except the $i$-th one.
Then, the $i$-th leave-one-out conformity score is computed on the training point that was left out, with
$
s^\loo_i = |Y_i - \widehat{f}_{-i}(X_i)|.
$ 
Letting $\widehat{f}$ denote the predictor obtained by training on all $n$ points, 
the jackknife prediction interval is then constructed by centering at the point prediction $\widehat{f}(X_{n + 1})$, and inflating it by the quantile of the leave-one-out scores. Specifically,\footnote{For a vector $v \in \mathbb{R}^m$, the scalar $\Quantile_{1-\alpha}(v)$ is the $(1-\alpha)$ empirical quantile of the entries that comprise $v$, i.e. the order statistic $v_{(k)}$ for $k = \lceil (1 - \alpha) m \rceil$.}
\[
\widehat{C}_n(X_{n+1}) = \left[ \widehat{f}(X_{n+1}) - \Quantile_{1-\alpha}(s^\loo_1, \dots, s^\loo_n), \; \widehat{f}(X_{n+1}) + \Quantile_{1-\alpha}(s^\loo_1, \dots, s^\loo_n) \right].
\]

In the setting of exchangeable data, existing theoretical results ~\citep{barber2021predictive,steinberger2023conditional} show that, if the model-fitting algorithm satisfies a mild form of out-of-sample stability with respect to leaving one data point out of the training set, then the jackknife achieves the desired coverage guarantee~\eqref{eq:coverage-desired} for any sample size $n$ up to a small additive correction to the miscoverage level $\alpha$. 
But how robust are these conclusions in the time series setting, where the data can be far from exchangeable?

\subsection{A motivating numerical experiment}
\label{sec:motivating}

To simulate a canonical time series, we generate our data $(X_i, Y_i)_{i \geq 1}$ according to a moving average (MA) process in various dimensions $d$ and assess the coverage of various predictive inference methods. To be concrete, fix some value of dimension $d \in \mathbb{N}$ and suppose the covariate and response space are given by $\mathcal{X} = \mathcal{Y} = \mathbb{R}^d$. Let $\{\omega_i\}_{i \geq 0}$ be i.i.d. standard Gaussians drawn from the distribution $\mathcal{N}(0,I_d)$. Construct $d$-dimensional covariates $\{X_i\}_{i = 1}^{n+1}$ according to the MA(1) process 
\[
 X_i \;=\; \omega_{i-1} + \omega_{i}, \quad \text{with} \; Y_i = X_{i+1} \quad \text{ for all } i=1,\dots,n+1.
\]
Our dataset is then given by $(X_i, Y_i)_{i = 1}^{n + 1}$. In our experiment, we use $n=200$.

We use \(2\)-nearest neighbors as the base predictor for training, and (since the response $Y$ is multivariate) we extend the definition of the jackknife to construct a ball, rather than an interval, around the prediction. Concretely, for each \(i\in[n]\), let \(\widehat f_{-i}\) denote the predictor trained on all training points except \(i\), and define the leave-one-out score
\[
s_i^{\mathrm{loo}}
:=
\bigl\|Y_i-\widehat f_{-i}(X_i)\bigr\|_2 .
\]
Let \(\widehat f\) be the predictor trained on all \(n\) training points, and define the leave-one-out quantile
\[
q_{\mathrm{loo}}
:=
\Quantile_{1-\alpha}
\bigl(s_1^{\mathrm{loo}},\ldots,s_n^{\mathrm{loo}}\bigr).
\]
The ``vanilla" jackknife prediction set is the \(\ell_2\)-ball centered at $\widehat f(X_{n + 1})$, viz.
\[
\widehat C^{\mathrm{jackknife}}_n(X_{n+1})
=
\left\{
y\in\mathbb R^d:
\|y-\widehat f(X_{n+1})\|_2\le q_{\mathrm{loo}}
\right\}.
\]
When \(d=1\), this reduces to the interval described in Section~\ref{sec:jk-intro}.

For comparison, we also include split conformal prediction (CP), which fits a \(2\)-nearest-neighbor predictor $\widehat{f}_{\mathsf{sp}}$ on the first $n/2$ points (i.e. using $50\%$ of the dataset for training) and computes the empirical \((1-\alpha)\)-quantile $q_{\mathsf{sp}}$ of scores on the remaining $n/2$ calibration points. The prediction set is then given by an \(\ell_2\)-ball of radius $q_{\mathsf{sp}}$ centered at the point prediction $\widehat{f}_{\mathsf{sp}}(X_{n + 1})$. 

\begin{figure}
    \centering
    \includegraphics[width=0.45\linewidth]{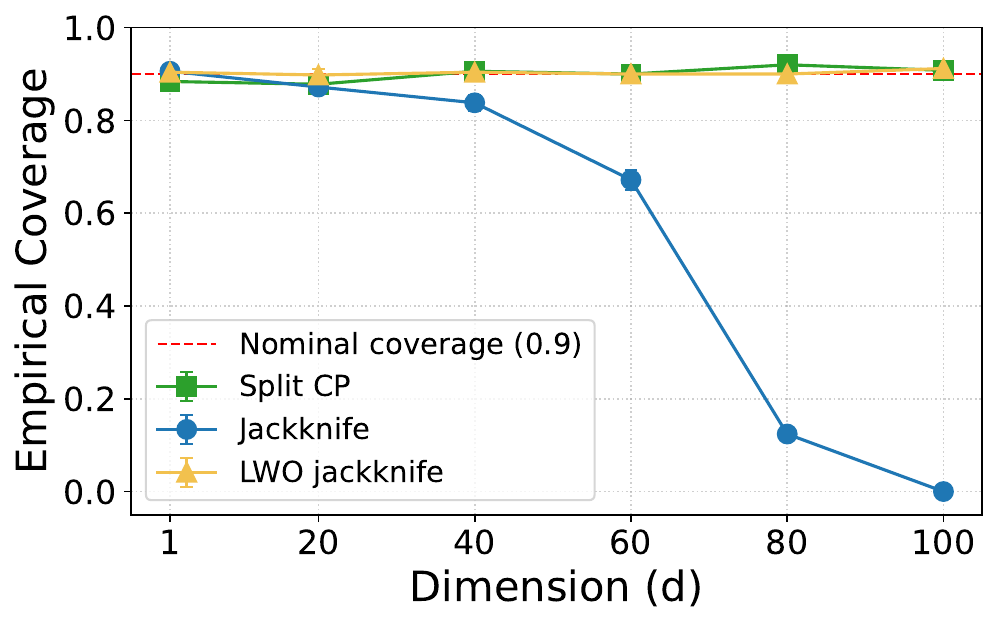}
    \includegraphics[width=0.45\linewidth]{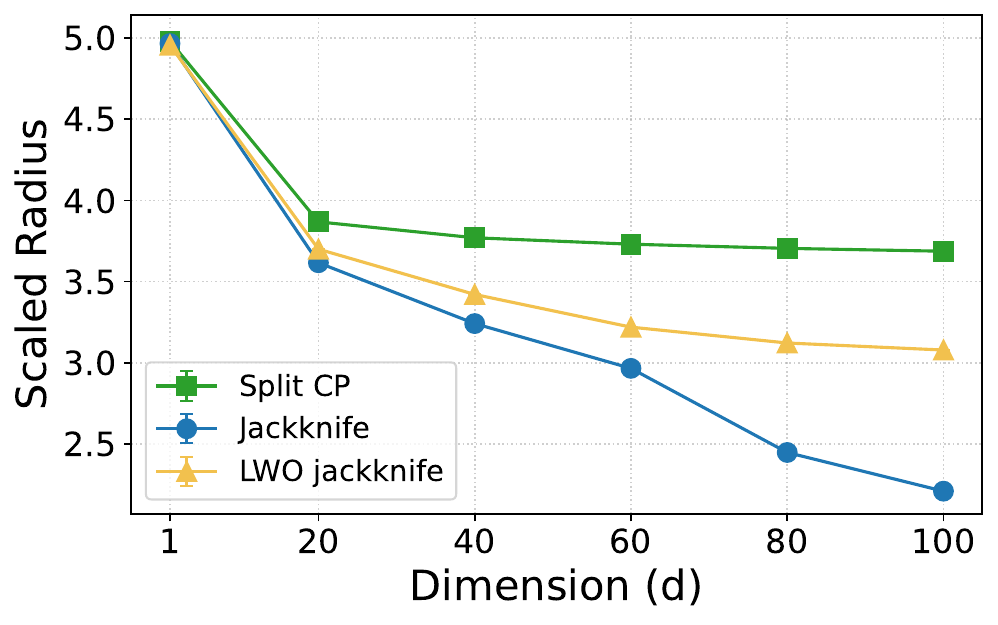}
    \caption{(Left) Empirical coverage on the multidimensional MA(1) process using \(2\)-nearest neighbors as the base predictor. Coverage is averaged over \(1000\) independent trials for split CP, the vanilla jackknife, and the LWO method. The nominal coverage level is \(90\%\) throughout. (Right)~Average radius of the prediction regions produced by split CP, the vanilla jackknife, and LWO, scaled by \(1/\sqrt{d}\) to account for the response dimension $d$.
}
    \label{counter}
\end{figure}

Figure~\ref{counter}(left) reports the empirical coverage of split CP and the vanilla jackknife. Even in this simple MA(1) setting, the vanilla jackknife can severely undercover, with the degradation becoming more pronounced in higher dimensions. By contrast, split CP maintains coverage close to the nominal level across dimensions, as predicted by prior theory~\citep{oliveira2024split,barber2025predictive}. In the same figure, we also plot the coverage of the ``leave-a-window-out" (LWO) method, which is the main methodological contribution of this paper. Evidently, the coverage behavior of this method is similar to that of split CP.
In Figure~\ref{counter}(right), we plot the radius of the prediction set for the three methods. Across dimensions, LWO produces consistently smaller prediction regions than split CP. At the same time, the size of the prediction region of LWO is similar to that of the vanilla jackknife when the latter exhibits reasonable coverage (i.e., for small dimensions).

\subsection{Related work}

For general background on conformal prediction, we refer the reader to the monographs~\citep{vovk2005algorithmic, shafer2008tutorial, angelopoulos2023conformal, angelopoulos2024theoretical} for introductions to both practical and theoretical aspects of the subject. In this section, we discuss two lines of work that are most closely related to our contributions.

\vspace{-2mm}
\paragraph{Predictive inference beyond data splitting.}
Split conformal prediction \citep{papadopoulos2002inductive,lei2018distribution} is typically the most computationally tractable among CP methods. However, it suffers from a widely-recognized ``split tax'': partitioning the data into a proper training set and a calibration set reduces the effective sample size available for model fitting, which can lead to unnecessarily wide prediction intervals and higher variability, especially in data-scarce regimes (see, e.g., Figure~\ref{counter}(right)). A large body of work has therefore sought to reduce or eliminate this statistical inefficiency by using the data more effectively for both training and calibration. Popular approaches in this direction include cross-conformal prediction \citep{vovk2015cross} and related aggregation methods such as the jackknife, jackknife+, jackknife-minmax, and CV+ methods of \citet{barber2021predictive}. Broadly speaking, these methods are based on leave-one-out or leave-fold-out residuals and are attractive because they often produce substantially tighter intervals than split CP while retaining valid coverage under exchangeability. More recent variants include multi-split CP \citep{solari2022multi}, out-of-bag CP \citep{linusson2020efficient}, nested conformal and quantile out-of-bag methods \citep{gupta2022nested}, and bootstrap-based procedures such as the jackknife+-after-bootstrap \citep{kim2020predictive}.

The jackknife validity result of \citet{barber2021predictive} relies on an out-of-sample stability condition, reflecting the fact that the leave-one-out residuals should not change too much when one training observation is removed. Stability-based arguments have also been used to obtain conditional or training-conditional forms of predictive validity~\citep{steinberger2023conditional,liang2025algorithmic} and to make conformal procedures computationally efficient~\citep{ndiaye2022stable}. Recent work shows that bagging can provide assumption-free stability even when the underlying base learner is unstable~\citep{soloff2024bagging}. More broadly, algorithmic stability has long played a central role in learning theory
\citep{bousquet2002stability}, including for dependent observations
\citep{mohri2010stability}. 

\vspace{-2mm}
\paragraph{Conformal prediction under dependence and nonexchangeability.}
Classical validity theory for conformal prediction relies on exchangeability, but many modern applications involve temporal dependence, distribution shift, adaptive data collection, or other violations of exchangeability. Several works have developed general frameworks for conformal prediction beyond exchangeability~\citep[e.g.][]{tibshirani2019conformal,podkopaev2021distribution,barber2023conformal,prinster2024conformal}. These works provide useful tools and perspectives for understanding validity when calibration and test points are not exchangeable, but they do not specifically address the failure of leave-one-out methods under local temporal dependence.

In the specific setting of dependent data exhibiting Markovian and/or mixing structure, \citet{chernozhukov2018exact} propose conformal methods based on approximate block exchangeability 
and \citet{nettasinghe2023extending} obtain exact validity for hidden Markov models on finite sample spaces. The utilization of (approximate) block-wise exchangeability in these methods is reminiscent of ideas from the classical ideas on resampling and bootstrapping with dependent data~\citep[see, e.g.][]{politis1994stationary,politis1991circular,lahiri2013resampling}.
Other papers~\citep{zheng2024conformal, oliveira2024split} analyze the robustness of split conformal prediction to violations of exchangeability observed in dependent data streams. Recent work by \citet{barber2025predictive} sharpens these guarantees and introduces the notion of switch coefficients to measure departures from exchangeability. A distinct line of work develops conformal methods specifically for online prediction~\citep{gibbs2021adaptive}, distribution drift~\citep{gibbs2024conformal}, and time-series forecasting~\citep{zaffran2022adaptive,angelopoulos2023conformalpid}. Other time-series conformal methods, such as EnbPI~\citep{xu2021conformal}, SPCI \citep{xu2023sequential} and KOWCPI~\citep{lee2025kernel} exploit the temporal structure of the residuals or conformity scores more directly. These insights have recently been extended to multistep and multidimensional responses for time-series~\citep{sun2024copula,xu2024conformal}, and spatial dependence structures ~\citep{jiang2026spatial}. While many of these methods are designed for the online (or streaming) setting, our focus (like that of e.g.,~\citet{chernozhukov2018exact,oliveira2024split,barber2025predictive}) is on the setting in which the dataset $(X_t, Y_t)_{t =1}^{n+1}$ is a fixed batch but may exhibit temporal dependence.

\vspace{-2mm}
\section{General setting and our method}
\label{sec:lwo}

In Section~\ref{sec:jk-intro}, we introduced the problem of predictive inference using memoryless predictors and the absolute residual score function, but given that memory-based prediction is the de facto method in time series analysis and conformity can be measured in various ways, we now present the more general setting that accommodates arbitrary score functions and predictors with memory.

\subsection{Memory-based prediction}

A predictor with memory $L$ produces at each time $t \geq 1$ a map from the $L$-sized window of past observations to $\mathcal{Y}$, with the goal of approximating the response $Y_t$. For convenience of notation, we therefore assume that the given data is indexed starting at $-(L - 1)$, so that we have access to\footnote{Note that this is simply a matter of indexing. If our data points $(X_i, Y_i)_{i = 1}^{n + 1}$ are indexed starting at $1$, then we may re-index as $(X_i, Y_i)_{i = -(L - 1)}^{n' + 1}$, where $n' = n - L$.}  the time series of $n + 1 + L$ data points $(X_i, Y_i)_{i = -(L - 1)}^{n + 1}$. As before, the $(n + 1)$-th point should be viewed as the test point.
Now for each $t \in [n + 1]$, define the lifted covariate $\textbf{X}_t=((X_k,Y_k)_{k=t-L}^{t-1},X_t)$. Any memory-$L$ predictor $f:(\mathcal{X}\times\mathcal{Y})^{L}\times\mathcal{X}\rightarrow\mathcal{Y}$ produces a mapping 
$
\textbf{X}_t \mapsto f(\textbf{X}_t)
$; let $\mathcal{F}$ denote the space of all such predictors. The closeness of this prediction to the desired response $Y_t$ is measured in terms of a conformity score function $\score: \Yspace \times \Yspace \to \mathbb{R}$.
For training, we have access to some prediction algorithm $\Alg: \bigcup_{k \in \mathbb{N}}\big((\mathcal{X}\times\mathcal{Y})^{L}\times\mathcal{X}\big)^k \to \mathcal{F}$ which takes as input $k$ samples of training data (for any natural number $k$) and outputs a particular memory $L$ predictor which is the best fit according to some criterion. Typically, one should think of $\Alg$ as being symmetric in its inputs.\footnote{Symmetry is assessed after lifting the covariates, so memory-based predictors can still be symmetric in this sense.}

\vspace{-2mm}
\paragraph{Example: Exogenous Time Series Regression.}
In this setting, the real response $Y_t$ depends on a separate vector of concurrent, exogenous covariates $X_t$. For example, we might wish to predict electricity consumption $Y_t$ at time $t$ based on observed weather conditions such as temperature, humidity, and other covariates, which we collect in $X_t$. It is reasonable to assume temporal dependence in this process---i.e., that $Y_t$ can be predicted based not only on contemporaneous exogenous information $X_t$ but also on covariates and responses in the $L$ past time points $t - L, \ldots, t - 1$. Thus, any predictor $f$ for $Y_t$ will map  $(\{X_{k}, Y_k\}_{k = t - L}^{t - 1}, X_t) \mapsto f(\{X_{k}, Y_k\}_{k = t - L}^{t - 1}, X_t)$. \hfill $\clubsuit$

\paragraph{Example: Autoregressive Forecasting.}
In many forecasting tasks, external covariates are unavailable, and prediction relies solely on past responses. In other words, the sequence is given by $(Y_{-(L-1)}, \ldots, Y_t)$ and the goal is to predict $Y_t$ from a window of past responses $Y_{t - L}, \ldots, Y_{t - 1}$. In order to write this in the notation of covariate-response pairs, we can follow the convention that $\mathcal{X} = \{ x_0\}$ for some fixed $x_0$, and any predictor for $Y_t$ may once again be viewed as a map from $(\mathcal{X} \times \mathcal{Y})^{L} \times \mathcal{X} \to \mathcal{Y}$ (where the predictor simply ignores $x_0$). \hfill $\clubsuit$

\medskip

Having given two canonical examples of memory-based prediction, we are now ready to define our predictive inference method.

\subsection{The leave-a-window-out (LWO) method} \label{sec:lwo-method}

Our leave-a-window-out (or LWO) method is given by the following steps:

\begin{enumerate}
    \item For each $k = 1, \ldots, n$:
\begin{itemize}
    \item Train the predictor $\widehat{f}_k$ on the subsequence of data points obtained by excluding the indices $k, \ldots, \min\{k+\tau,n\}$. Specifically, compute
    \begin{subequations} \label{eq:lwo-def}
    \begin{align} \label{eq:lwo-prediction-def}
    \widehat{f}_k = &\Alg \Big[ 
    \underbrace{(\textbf{X}_{1},Y_{1}), \ldots, (\textbf{X}_{k-1}, Y_{k-1})}_{\text{Pre-window sequence}}, 
    \underbrace{(\textbf{X}_{k+\tau+1}, Y_{k+\tau+1}), \ldots, (\textbf{X}_{n}, Y_{n})}_{\text{Post-window sequence}} 
    \Big],
    \end{align}
    where the post-window sequence is empty for $k \geq n - \tau$.
    
    \item Compute the score 
    \begin{align} \label{eq:lwo-score-def}
        s_k = \score(Y_k, \widehat{f}_k(\textbf{X}_k)).
    \end{align}.
\end{subequations}
\end{itemize}

\item Let $\widehat{f} = \Alg[(\bX_1, Y_1), \ldots, (\bX_n, Y_n)]$ denote the predictor obtained by training on all $n$ training data points. We then output the prediction set
\begin{equation}
 C^{\text{lwo}}_{\level} = \left\{y \in \Yspace : \score(y, \widehat{f}(\textbf{X}_{n+1})) \leq \Quantile_{1 - \alpha} ( s_1, \ldots, s_{n} ) \right\}.
 \label{lwo_interval}
\end{equation}
\end{enumerate}

\begin{figure}[t]
\centering
\resizebox{0.85\linewidth}{!}{%
\begin{tikzpicture}[
    >=Latex,
    font=\small,
    cell/.style={
        draw,
        line width=0.65pt,
        minimum width=1.25cm,
        minimum height=0.95cm,
        align=center,
        inner sep=1pt,
        font=\small
    },
    train/.style={cell, fill=green!25},
    proxy/.style={cell, fill=orange!45},
    omit/.style={cell, fill=white},
    brace/.style={
        decorate,
        decoration={brace, amplitude=5pt},
        line width=0.8pt
    }
]

\node[train] (z1)  at (0.00,0)  {\(\mathbf X_1\)\\[-1pt]\(Y_1\)};
\node[train] (z2)  at (1.25,0)  {\(\mathbf X_2\)\\[-1pt]\(Y_2\)};
\node[train] (z3)  at (2.50,0)  {\(\cdots\)};
\node[train] (z4)  at (3.75,0)  {\(\mathbf X_{k-1}\)\\[-1pt]\(Y_{k-1}\)};

\node[proxy] (zk)  at (5.00,0)  {\(\mathbf X_k\)\\[-1pt]\(Y_k\)};

\node[omit]  (z5)  at (6.25,0)  {\(\mathbf X_{k+1}\)\\[-1pt]\(Y_{k+1}\)};
\node[omit]  (z6)  at (7.50,0)  {\(\cdots\)};
\node[omit]  (z7)  at (8.75,0)  {\(\mathbf X_{k+\tau}\)\\[-1pt]\(Y_{k+\tau}\)};

\node[train] (z8)  at (10.00,0) {\(\mathbf X_{k+\tau+1}\)\\[-1pt]\(Y_{k+\tau+1}\)};
\node[train] (z9)  at (11.25,0) {\(\cdots\)};
\node[train] (z10) at (12.50,0) {\(\mathbf X_n\)\\[-1pt]\(Y_n\)};

\draw[brace]
    ($(z1.north west)+(0,0.20)$) --
    ($(z4.north east)+(0,0.20)$)
    node[midway, yshift=0.43cm, font=\small\bfseries]
    {Pre-window training set};

\draw[brace]
    ($(z8.north west)+(0,0.20)$) --
    ($(z10.north east)+(0,0.20)$)
    node[midway, yshift=0.43cm, font=\small\bfseries]
    {Post-window training set};

\draw[brace, decoration={brace, mirror, amplitude=5pt}]
    ($(z5.south west)+(0,-0.22)$) --
    ($(z7.south east)+(0,-0.22)$)
    node[midway, yshift=-0.70cm, font=\small\bfseries]
    {Window left out};

\node[font=\small\bfseries, anchor=east] (proxyword)
    at (3.53,-1.4) {Proxy};

\node[font=\small\bfseries, anchor=west] (testword)
    at (3.40,-1.38) {test};

\node[font=\small\bfseries, anchor=west] (pointword)
    at ($(testword.east)+(-0.13,-0.01)$) {point};

\draw[-Latex, line width=0.75pt]
    (testword.north) -- ($(zk.south)+(0.02,-0.04)$);

\end{tikzpicture}%
}

\caption{Illustration of how the LWO score $s_k$ is computed. We leave out a window of length $\tau$ starting at (and including) $(\bX_{k+1}, Y_{k+1})$, which is denoted by the white blocks. The predictor $\widehat{f}_k$ is trained on the green points (i.e. all remaining data except $(\bX_k, Y_k)$), and the score $s_k$ is obtained by evaluating the trained predictor $\widehat{f}_k$ on $(\bX_k, Y_k)$. }
\label{fig:lwo-score}
\end{figure}
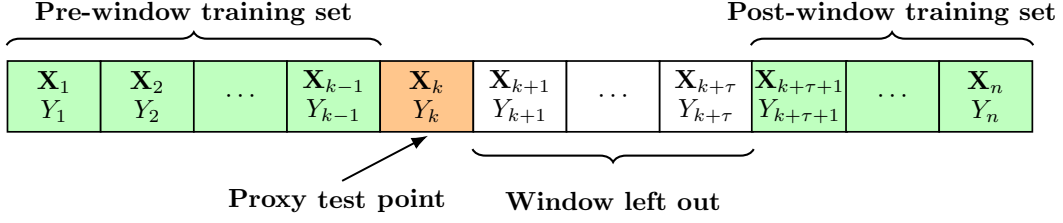

To describe our procedure in words, note that score $s_k$ is computed by treating the data at time index $k$ as a proxy for the test point: it is generated by comparing the prediction $\widehat{f}_k(\bX_k)$ to the corresponding output $Y_k$, just like in the vanilla jackknife. The main difference is in how we compute the predictor $\widehat{f}_k$ in Eq.~\eqref{eq:lwo-prediction-def}. Unlike in the vanilla jackknife, our LWO predictor is trained on all the data except for a \emph{window} of points corresponding to time indices $k, \ldots, k + \tau$. Figure~\ref{fig:lwo-score} illustrates how the score is computed; note that the training set now has both a ``pre-window" and ``post-window" component. The post-window component may still use lagged information from the window (through the lifted covariates). Note that the typical amount of computation needed for LWO is comparable to the vanilla jackknife; indeed, we recover the vanilla jackknife described in Section~\ref{sec:jk-intro} when $\tau = 0$ (if we use the residual score, $\score(y, \widehat{f}(\textbf{X}_{n+1})) = |y - \widehat{f}(\textbf{X}_{n+1})|$).

Note that our algorithm involves a hyperparameter, which is the window size $\tau$. A natural  heuristic to choose \(\tau\) is to have it be at least as large as the effective dependence range that we expect in the problem. In practice, this can be guided by domain knowledge, autocorrelation diagnostics, or validation-based selection over a small grid of candidate window lengths.

\subsection{Overview of results}

Now that we have defined our LWO method, we use the remaining sections to provide theoretical and empirical results on its performance.
Specifically:

\begin{enumerate}
    \item 
    We develop an analysis framework for LWO showing that its coverage loss is bounded under a natural stability assumption; we introduce this assumption along with some other useful notions in Section~\ref{sec:masking-stability}. Later in Section~\ref{sec:theory}, we introduce the \emph{cyclic embedding coefficient} of a process to measure how close the data sequence $(\bX_1,Y_1), \ldots, (\bX_{n+1}, Y_{n+1})$ is to a longer cyclically exchangeable sequence, and prove two finite-sample coverage bounds for LWO in terms of this coefficient. Our first result, Theorem~\ref{thm:cyc}, is based on embedding the original data sequence, and our second result, Theorem~\ref{thm:cyc_stb}, is based on embedding a carefully modified version of the data sequence.

    \item 
    We bound the cyclic embedding coefficient under standard dependence conditions. We first study
approximately Markovian processes in Section~\ref{sec:theory-markov} and control the cyclic embedding coefficient using mixing properties of the process. Then, we study general (not necessarily Markovian) mixing processes in Section~\ref{sec:algstb} and relate the resulting cyclic embedding coefficient to the switch coefficients of the process. Together, these results, presented in Corollaries~\ref{cor:cyc} and~\ref{cor:switch} respectively, show that LWO produces valid prediction intervals despite natural notions of temporal dependence.

    \item 
    In Section~\ref{sec:experiments}, we demonstrate the empirical benefits of LWO on simulated and real time series. The experiments illustrate that LWO achieves valid coverage under dependence even when the vanilla jackknife fails, while remaining substantially more data-efficient than split conformal prediction.
\end{enumerate}

Section~\ref{sec:proofs} collects proofs of all our main results, with technical lemmas postponed to Appendix~\ref{sec:app-proofs}. We conclude with a discussion of open questions in Section~\ref{sec:disc}, and present additional experimental results in Appendix~\ref{sec:add_expt}.

\section{Theoretical results} \label{sec:theory}

In this section, we state theoretical results for LWO under natural dependence assumptions. Along the way, we introduce new coefficients to measure the extent to which cyclic exchangeability is violated by a dependent sequence.

 \subsection{Preliminary setup} \label{sec:masking-stability}

 Before we present our theoretical results, we present a useful reduction, some notation, and a natural stability assumption.
 
\subsubsection{Reduction to zero memory predictors, dummy data points, and masking} \label{sec:reduction}

Note that our setting can be equivalently described as follows. Let $\mathbb{X} := (\mathcal{X} \times \mathcal{Y})^{L} \times \Xspace$. We have the stochastic process $\bZ = (Z_1, \ldots, Z_{n + 1})$, where $Z_t = (\bX_t, Y_t)$ is a random variable taking values in $\Zspace := \mathbb{X} \times \Yspace$. Each memory-$L$ predictor can be viewed as a function $f: \bbX \to \Yspace$ that produces, for each $t \geq 1$, a prediction $f(\bX_t)$ for $Y_t$. It thus suffices to consider memoryless prediction (i.e., standard regression functions $\bX_t \mapsto f(\bX_t)$) for the stochastic process $(Z_t)$. Given this reduction, the setting for the rest of the paper is as follows. Given $Z_1, \ldots, Z_n$ and $\bX_{n+1}$, our goal is to construct a prediction interval for the test response $Y_{n + 1}$ with coverage assessed marginally over $\bZ$.

It is useful to also define a notion of \emph{dummy data points}. A dummy data point is denoted by $\star$, and we define the augmented data space $\mathcal{Z}_\star := \mathcal{Z} \cup \{ \star \}$. Analogously, the augmented covariate space is given by $\bbX_* = \bbX \cup \star$ and the augmented response space by $\Yspace_* = \Yspace \cup \star$. We abuse notation slightly and say that if $Z = (\bX, Y)$ is a dummy data point, then $\bX = Y = \star$. 
We adopt the following conventions for dummy points. First, any predictor satisfies $f(\star)=\star$, and $\mathcal{F}_\star$ denotes the augmented space of predictors. Second, the training algorithm is extended to $\Alg:\bigcup_{k\ge1}(\mathcal Z_\star)^k \to \mathcal F_{\star}$ with the convention that dummy points in the training input are ignored by the training algorithm, and $\Alg$ always outputs a dummy point if it is trained only on dummy points. 
Finally, the score function is extended so that any comparison involving a dummy point contributes zero score, i.e. $\score(\star, f(\star)) = 0$ for any $f \in \mathcal{F}_{\star}$.

Having dummy data points then allows us to define a \emph{masking} operation.
\begin{definition}[The masking operation] 
\label{def:mask}
Fix any $m\geq 1$ and any $\tau\in\{0,\dots,m-1\}$, and let $\mathbf{w}=(w_1,\dots,w_m)$ be any vector. For each $k\in\{-\tau,\dots,m-1\}$, we define a \emph{masked} version of the vector $\mathbf{w}$ as
\[\big(\mathrm{M}_{k,\tau}(\mathbf{w})\big)_i 
= \begin{cases} \star, & i\in\{k+1,\dots,k+\tau\},\\
w_i, & i\in[m]\setminus \{k+1,\dots,k+\tau\}.\end{cases}\]
In other words, the masked vector $\mathrm{M}_{k,\tau}(\mathbf{w})$ is obtained from $\mathbf{w}$ by replacing entries $w_i$, for any $i=k+1,\dots,k+\tau$, with $\star$.
\end{definition}

In particular, for $1 \leq k \leq m-\tau-1$, the masked vector can be written as
\[
\mathrm{M}_{k,\tau}(\mathbf{w})
=
\bigl(w_1,\ldots,w_k,\underbrace{\star,\ldots,\star}_{\tau\text{ times}}, 
      w_{k+\tau+1},\ldots,w_m\bigr),
\]
for $m -\tau \leq k \leq m - 1$ as
\[
\mathrm{M}_{k,\tau}(\mathbf{w})
=
\bigl(w_1,\ldots,w_k, \underbrace{\star,\ldots,\star}_{(m - k)\text{ times}}\bigr).
\]
and for $-\tau \leq k \leq 0$,
\[
\mathrm{M}_{k,\tau}(\mathbf{w})
=
\bigl(\underbrace{\star,\ldots,\star}_{(k+\tau)\text{ times}}, 
      w_{k+\tau+1},\ldots,w_m\bigr).
\]
By this convention, we have $\mathrm{M}_{-\tau,\tau}(\mathbf{w}) = \mathbf{w}$. 

Note that masking can be applied in composition. In particular, $
\mathrm{M}_{k',\tau'}\bigl(\mathrm{M}_{k,\tau}(\mathbf{w})\bigr)
$
is the vector obtained by masking the two blocks beginning at indices \(k+1\) and \(k'+1\). In particular, once an entry is masked, it remains masked under any subsequent masking operation. Therefore, the composition is invariant to the order in which the two masking operations are applied.

With this notation in hand, we can now write the LWO method as applied directly to the sequence $\bZ=(Z_1,\ldots,Z_{n+1})$ with $Z_t=(\bX_t,Y_t)\in\Zspace$. Fix a window length $\tau\ge 0$. For each $k=1,\ldots,n$, we train a predictor by leaving out (or equivalently, masking) the block of indices $k,\ldots,\min\{k+\tau,n\}$, namely
\[
\widehat f_k
\;=\;
\Alg\big[ \mathrm{M}_{k  -1, \tau + 1}(Z_1,\ldots,Z_n) \big],
\]
and compute the corresponding score
\[
s_k = \score(Y_k,\widehat f_k(\bX_k)).
\]
Let $\widehat f=\Alg[ Z_1,\ldots,Z_n]$ denote the predictor trained on all $n$ training points. The LWO prediction set for the test response at time $n+1$ is then
\[
C^{\mathrm{lwo}}_\alpha[Z_1, \ldots, Z_n](\bX_{n+1})
\;=\;
\Big\{y\in\Yspace:\ \score(y,\widehat f(\bX_{n+1}))\le \Quantile_{1-\alpha}(s_1,\ldots,s_{n})\Big\},
\]
which reduces to the vanilla (leave-one-out) jackknife when $\tau=0$. We also define an inflated prediction interval
\[
C^{\mathrm{lwo},t}_\alpha[Z_1, \ldots, Z_n](\bX_{n+1})
\;=\;
\Big\{y\in\Yspace:\ \score(y,\widehat f(\bX_{n+1}))\le \Quantile_{1-\alpha}(s_1,\ldots,s_{n})+t\Big\},
\]
where the quantile is inflated by $t \geq 0$. We often suppress the dependence on $Z_1, \ldots, Z_n$ and use the shorthand $C^{\mathrm{lwo}}_\alpha(\bX_{n+1})$ and  $C^{\mathrm{lwo},t}_\alpha(\bX_{n+1})$ for these prediction sets.

As a final piece of notation, let us introduce our convention for subvectors. If $\mathbf{w} = (w_1, \ldots, w_m)$ is a length $m$ vector and $1 \leq a \leq b \leq m$, we let $\mathbf{w}_{a:b} = (w_a, \ldots, w_b)$ denote the subvector whose indices range from (and include) $a$ to $b$. For example, the vector $\bZ_{1:n}$ denotes the training sequence.
\medskip

 Since the method involves computing each score by leaving out (or masking) a $\tau$-length block from the training data, the LWO method is particularly suited for problems in which the $(\score, \Alg)$ pair satisfies a natural notion of stability. We introduce this notion next.

\subsubsection{Out-of-sample stability to masking a random training block}
\label{sec:oos_stab}
The following definition is a general notion of out-of-sample (OOS) stability, and is stated for a general random sequence $\bZ' \in \Zspace^{n + 1}$:
\begin{definition}
\label{def:stab}
For a pair of scalars $(\nu, t)$ and a natural number $\ell$, the pair $(\score, \Alg)$ is said to be OOS stable with respect to masking a random training block of length $\ell$ from $\bZ' = (\bX'_i, Y'_i)_{i = 1}^{n + 1}$ if for $K \sim \mathsf{Unif}(\{1-\ell,\ldots,n - \ell\})$, we have
\begin{align*}
\mathbb{P} \Bigg\{ \Big| \score(Y'_{n + 1}, \Alg[ \mathrm{M}_{K, \ell} (Z'_1, \ldots, Z'_n)](\bX'_{n + 1})) - \score(Y'_{n + 1}, \Alg[Z'_1, \ldots, Z'_n ](\bX'_{n + 1})) \Big| \leq t \Bigg\} \geq 1-\nu.
\end{align*}
\end{definition}

In other words, the score computed on the $(n + 1)$-th (test) data point is stable to leaving out a uniformly random block of length $\ell$ from the $n$ (training) data points.\footnote{By the definition of the masking operation, we also include cases in which blocks of size $\leq \ell$ are left out at the start of the training sequence, corresponding to when $K$ takes negative values.} Note that this is an out-of-sample (or test-time) notion, since stability is always measured with respect to the test-point prediction and not the training point predictions.
Other notions of stability have been defined for algorithms on dependent data~\citep{mohri2010stability}. One such definition is uniform stability~\citep{mohri2010stability}, which measures \emph{in-sample} stability under deletion of any single training point.  Definition~\ref{def:stab} is weaker, since it is both an out-of-sample notion of stability, and averaged over the deletion of all contiguous blocks of length $\ell$.

\paragraph{Examples of stable algorithms.} To be concrete, let us give some concrete examples of stable algorithms when the score function is given by the absolute residual \(\score(y,\widehat{y})=|y-\widehat{y}|\). First, it is known that ridge regression, kernel ridge regression, and elastic net satisfy in-sample uniform stability under standard boundedness and strong-convexity assumptions~\citep{bousquet2002stability,mohri2010stability}. For these algorithms, it can be verified that the $(\score, \Alg)$ pair is OOS stable with respect to any $\bZ$ in the sense of Definition~\ref{def:stab} with
\begin{align} \label{eq:ridge-stable}
\nu=0 
\quad \text{and} \quad
t =  C \cdot \sum_{r=0}^{\ell-1}\frac{1}{n-r}
\le C\log\!\left(\frac{n}{n-\ell}\right),
\end{align}
where $C > 0$ depends on the boundedness and strong convexity parameters. A second example (which is not in-sample stable) is the $k$-nearest neighbor algorithm, a particular instance of which was used in our experiments in Figure~\ref{counter}. In this case, the prediction at the test point changes only if the deleted block intersects the $k$ nearest neighbors of the test covariate, and so OOS stability in the sense of Definition~\ref{def:stab} is satisfied with
\begin{align} \label{eq:NN-stable}
\nu = \frac{k\ell}{n} \quad \text{and} \quad t=0.
\end{align}
To give a third example, we note that if $\Alg$ is obtained by bagging bounded base predictors, then OOS stability in the sense of Definition~\ref{def:stab} is satisfied even when the underlying base learner is unstable. In particular, the guarantees of \citet{soloff2024bagging}, which show OOS stability to leaving one training point out, can be straightforwardly extended via application of the triangle inequality to show stability of bagged predictors in the sense of Definition~\ref{def:stab}.

Having introduced our method as well as a suitable stability definition on the $(\score, \Alg)$ pair, we are now in a position to provide theoretical guarantees under certain assumptions on the process~$\bZ$.

\subsection{Guarantees when $\bZ$ is approximately Markovian and mixing}
\label{sec:theory-markov}

A canonical class of time series models can be described as Markov chains---these include autoregressive processes and state space models, among others. In this section, we show that LWO attains approximate coverage for this class of processes provided that the process is suitably mixing and the $(\score, \Alg)$ pair is OOS stable with respect to $\bZ$.
To arrive at this conclusion, we introduce a new conceptual tool: embedding the given time series $\bZ$ into a longer cyclically exchangeable sequence.

\subsubsection{Coverage of LWO when $\bZ$ is (approximately) cyclically embeddable}
\label{sec:condmix}

To understand cyclical embeddings of $\bZ$, we first introduce cyclic exchangeability. We use the notation $\overset{\textnormal{d}}{=}$ to denote equality in distribution.
\begin{definition}[Cyclic exchangeability]
 A sequence $\widetilde{\bZ} = (\widetilde{Z}_1,\dots,\widetilde{Z}_{m}) \in \Zspace_*^{m}$ is said to satisfy cyclic exchangeability if
\[\bZtil \overset{\textnormal{d}}{=} (\widetilde{Z}_i,\dots,\widetilde{Z}_{m},\widetilde{Z}_1,\dots,\widetilde{Z}_{i-1})\]
for all $i\in[m]$. Let $\mathcal{S}(m)$ denote the collection of all cyclically exchangeable sequences of length $m$.
\end{definition}

Our first result is that if $\bZ$ can be embedded into a cyclically exchangeable sequence of length $n + 1 + \tau$, then we have approximate validity for the inflated LWO prediction set with window size~$\tau$. 

\begin{proposition}
\label{prop:cyc}
Suppose there exists $\bZtil \in \mathcal{S}(n + 1 + \tau)$ such that $\bZ \overset{\textnormal{d}}{=} \bZtil_{1:(n+1)}$.
If $\Alg$ is symmetric and the pair $(\score, \Alg)$ is OOS stable with respect to $\bZ$ in the sense of Definition~\ref{def:stab} with parameters $(\nu, t, \tau + 1)$, then
\[
\mathbb{P}\left(Y_{n+1}\in C^{\jk,t}_{\level}(\bX_{n+1})\right) \geq 1-\alpha -2\sqrt{\nu} - \frac{\tau+1}{n}.
\]
\end{proposition}
We prove Proposition~\ref{prop:cyc} in Section~\ref{sec:pf-prop1}.
Notably, if $\bZ$ is an exchangeable sequence, then it is itself cyclically exchangeable, and can therefore be trivially embedded into a cyclically exchangeable sequence of length $n + 1$. Proposition~\ref{prop:cyc} then implies that the vanilla jackknife (or equivalently, LWO with $\tau=0$) satisfies approximate validity, thereby recovering the result of~\cite{barber2021predictive}.
Another example of a cyclically embeddable process is a moving average process with memory $q$, as simulated in Figure~\ref{counter}. If $\bZ$ is drawn from a stationary $\textnormal{MA}(q)$ process, then it is embeddable into a cyclically exchangeable process of length at least $n + 1 + q$. Proposition~\ref{thm:cyc} then guarantees that LWO with a window size of at least $q$ will achieve approximate validity. As a related example, for the process in Figure~\ref{counter} (which is cyclically embeddable\footnote{Even though the covariates follow MA(1), the response adds a point in the future, thereby increasing memory by one.} into a sequence of length $n + 1 + q$ for any $q \geq 2$), we used $\tau = 5$ when implementing LWO; clearly, it remains valid for this choice of window size.

While Proposition~\ref{thm:cyc} covers exact cyclical embedding, it can be naturally extended to the situation in which such an embedding can only be performed \emph{approximately}, as articulated by the following coefficient.

\begin{definition}[Cyclic embedding coefficient]
For a nonnegative integer $a$, define the cyclic embedding coefficient of $\bZ$ with lengthening $a$ as
\[
\rho_{a}(\bZ):= \inf_{\widetilde{\bZ} \in \mathcal{S}(n + 1 + a)} \mathrm{d}_{\mathrm{TV}}(\bZ,\widetilde{\bZ}_{1:(n+1)}).
\]
\end{definition}

In Section~\ref{sec:pf-thm1}, we combine the definition of this coefficient with Proposition~\ref{prop:cyc} to arrive at the following theorem.

\begin{theorem}
\label{thm:cyc}
If $\Alg$ is symmetric and the pair $(\score, \Alg)$ is OOS stable with respect to $\bZ$ in the sense of Definition~\ref{def:stab} with parameters $(\nu, t, \tau + 1)$, then
\[
\mathbb{P}\left(Y_{n+1}\in C^{\jk,t}_{\level}(\bX_{n+1})\right) \geq 1-\alpha-\rho_{\tau}(\bZ)-\frac{\tau+1}{n}-2\sqrt{\nu+\rho_{\tau}(\bZ)}.
\]   
\end{theorem}
\noindent Theorem~\ref{thm:cyc} is proved in Section~\ref{sec:pf-thm1}.

While Theorem~\ref{thm:cyc} shows that the loss of coverage can be bounded by the cyclic embedding coefficient $\rho_\tau(\bZ)$, we would ideally be able to show that the cyclic embedding coefficient is related to standard mixing coefficients, in order to be able to interpret the implications of this result. We relate these two notions in the next section under the additional assumption that $\bZ$ is approximately Markovian.

\subsubsection{Relating the cyclic embedding coefficient to mixing under Markov property}

The following definition is a standard measure of dependence in a time series~\citep[e.g.][]{doukhan1995mixing}.
\begin{definition}[$\beta$-mixing coefficient]
For a sequence $\bZ\in\mathcal{Z}^{n+1}$, the $\beta$-mixing coefficient with lag $\tau$ is defined as
\[\beta(\tau):=\max _{1 \leq k \leq n-\tau} \mathrm{d}_{\mathrm{TV}}\left(\left(Z_1, \ldots, Z_k, Z_{k+\tau+1}, \ldots, Z_{n+1}\right),\left(Z_1, \ldots, Z_k, Z_{k+\tau+1}^{\prime}, \ldots, Z_{n+1}^{\prime}\right)\right),\]
where $\bZ'=(Z_1',\ldots,Z_{n+1}')\in\mathcal{Z}^{n+1}$ denotes an i.i.d. copy of $\bZ$.
\end{definition}
As mentioned before, we will also consider sequences that are approximately Markovian by measuring dependence in terms of the following $\beta$-conditional-mixing coefficient:
\begin{definition}[$\beta$-conditional-mixing coefficient]
For a sequence  
$\bZ\in\mathcal{Z}^{n+1}$, the $\beta$-conditional-mixing coefficient with lag $\tau$ is defined as
\[\beta^*(\tau):=\max _{1 \leq k \leq n-\tau} \mathrm{d}_{\mathrm{TV}}\left(\bZ,\left(Z_1^{\prime}, \ldots, Z_k^{\prime},Z_{k+1},\ldots,Z_{k+\tau}, Z_{k+\tau+1}^{\prime\prime}, \ldots, Z_{n+1}^{\prime\prime}\right)\right),\]
where $(Z_1',\ldots,Z_{k}')$ is sampled from the distribution $\mathbb{P}_{(Z_1,\ldots,Z_k)|(Z_{k+1},\ldots,Z_{k+\tau})}(\; \cdot\; | Z_{k+1},\ldots,Z_{k+\tau})$, and
$(Z''_{k + \tau + 1},\ldots,Z_{n+1}'')$ is sampled from the distribution $\mathbb{P}_{(Z_{k+\tau+1},\ldots,Z_{n+1})|(Z_{k+1},\ldots,Z_{k+\tau})}(\;\cdot\; | Z_{k+1},\ldots,Z_{k+\tau})$.
\end{definition}
Our definition of a conditional $\beta$-mixing coefficient does not seem to have appeared in the literature, although related notions (e.g., by~\citet{prakasa2009conditional}) are common. To develop intuition for this definition, suppose $\bZ$ is such that for any decomposition $\bZ = (Z_A, Z_B, Z_C)$ where the block $B$ has length $\tau$, we have the Markovian relation $Z_A \indpt Z_C \; | \; Z_B$. In this case, we would have $\beta^*(\tau) = 0$. Consequently, any mixing Markov chain has both $\beta(\tau)$ and $\beta^*(\tau)$ small provided $\tau$ is larger than the mixing time of the chain.

The following proposition then shows that the cyclic coefficient $\rho_\tau(\bZ)$ can be bounded by combining the above mixing coefficients.
\begin{proposition}
\label{prop:cycliccoef}
Suppose $\bZ$ is a stationary sequence, with $\beta$-conditional-mixing coefficient $\beta^*(\tau)$ and $\beta$-mixing coefficient $\beta(\tau)$. Then
$$
\rho_{\tau}(\mathbf{Z}) \leq  2\beta(\tau) + 2\beta^*(\tau)+\frac{4\tau}{n+\tau+1}.
$$
\end{proposition}

Proposition~\ref{prop:cycliccoef} is proved in Section~\ref{sec:pf-prop2}.
Combining Theorem~\ref{thm:cyc} and Proposition~\ref{prop:cycliccoef} immediately yields the following corollary.

\begin{corollary}
\label{cor:cyc}
If $\Alg$ is symmetric and the pair $(\score, \Alg)$ is OOS with parameters $(\nu, t, \tau + 1)$, and $\bZ$ is stationary, $\beta$-conditional mixing and $\beta$-mixing, then
\[
\mathbb{P}\left(Y_{n+1}\in C^{\jk,t}_{\level}(\bX_{n+1})\right) \geq 1-\alpha-3\sqrt{\nu+2\beta(\tau)+2\beta^*(\tau)+\frac{5\tau+1}{n}} .
\]   
\end{corollary}

Taking stock, Corollary~\ref{cor:cyc} shows that the validity of the LWO method is governed by two sources of dependence at scale \(\tau\). The first source \(\beta^*(\tau)\) measures how close to Markovian the sequence is when we allow ourselves to condition on blocks of length $\tau$, while the second source \(\beta(\tau)\) measures whether the sequence is rapidly mixing. In particular, if \(\tau\) is chosen large enough that both quantities are small, then the LWO method achieves approximate coverage. As previously mentioned, the prototypical example where this is possible for small $\tau$ is a fast mixing Markov chain. While this is an important class of stochastic processes, the next section provides a result for a larger class of sequences.

\subsection{Guarantees when $\bZ$ is only mixing}
\label{sec:algstb}

In practice, it may be the case that $\bZ$ is a mixing process without being approximately Markovian.  Is the LWO method still an approximately valid procedure in such a setting? In this section, we answer this question in the affirmative provided the $(\score, \Alg)$ pair satisfies a slightly stronger stability condition. In order to prove our eventual result (given in Corollary~\ref{cor:switch}), we proceed via cyclically embedding \emph{masked} versions of $\bZ$ instead of $\bZ$ itself.

\subsubsection{Coverage of LWO when masked $\bZ$ is cyclically embeddable}

To state our result, we require the notion of switch coefficients, which were introduced in~\cite{barber2025predictive} as a measure of departure from exchangeability; our definition is a slight modification of theirs.

\begin{definition}[Switch coefficients] \label{def:switch-coeff}
For each \(k\in\{-\tau,\ldots,n\}\), define
\[
\Delta^0_{k,\tau}(\bZ):=
\begin{cases}
(Z_1,\ldots,Z_k, Z_{k+\tau+1},\ldots,Z_{n+1}),
& 1 \le k \le n-\tau, \\[4pt]
(Z_1,\ldots,Z_k),
& n-\tau+1 \le k \le n, \\[4pt]
(Z_{k+\tau+1},\ldots,Z_{n+1}),
& -\tau \le k \le 0,
\end{cases}
\]
and
\[
\Delta^1_{k,\tau}(\bZ):=
\begin{cases}
(Z_{n+2-k},\ldots,Z_{n+1}, Z_1,\ldots,Z_{n-k-\tau+1}),
& 1 \le k \le n-\tau, \\[4pt]
(Z_{n+2-k},\ldots,Z_{n+1}),
& n-\tau+1 \le k \le n, \\[4pt]
(Z_1,\ldots,Z_{n-k-\tau+1}),
& -\tau \le k \le 0.
\end{cases}
\]

The \(k\)-th switch coefficient of \(\bZ\) with lag \(\tau\) is defined as
\[
\Psi_{k,\tau}(\bZ)
:=
\mathrm d_{\mathrm{TV}}\!\bigl(\Delta^0_{k,\tau}(\bZ),\,\Delta^1_{k,\tau}(\bZ)\bigr).
\]

The average switch coefficient is given by
\[
\overline{\Psi}_\tau(\bZ)
:=
\frac{1}{n+\tau+1}\sum_{k=-\tau}^{n}\Psi_{k,\tau}(\bZ).
\]
\end{definition}

Our main result of this section expresses the loss of coverage of the LWO method using the coefficient $\rho_\tau(\mathrm{M}_{K,\tau}(\bZ))$ and the switch coefficient $\overline{\Psi}_{\tau}(\bZ)$. Recall our notion of OOS stability from Definition~\ref{def:stab}. 

\begin{theorem}
\label{thm:cyc_stb}
Suppose $\Alg$ is symmetric and let $K\sim\mathsf{Unif}(\{-\tau,\ldots,n\})$ be drawn independently of the data. Now suppose that the pair $(\score, \Alg)$ is OOS stable with respect to $\bZ$ with parameters $(\nu, t, \tau)$, OOS stable with respect to $\mathrm{M}_{K,\tau}(\bZ)$ with parameters $(\nu, t, \tau+1)$, and OOS stable with respect to $\mathrm{M}_{K,\tau+1}(\bZ)$ with parameters $(\nu, t, \tau)$. Then we have
\begin{align*}
\mathbb{P}(Y_{n+1}\in C_{\alpha}^{\jk,3t}(\bX_{n+1}))
\geq 1-\alpha-3\sqrt{\nu+\frac{2\tau+2}{n}+\rho_{\tau}(\mathrm{M}_{K,\tau}(\bZ))}-2\sqrt{\nu+\frac{\tau}{n}+\frac{n+\tau+1}{n} \cdot \overline\Psi_\tau(\bZ)}.    
\end{align*}
\end{theorem}
We provide the proof of Theorem~\ref{thm:cyc_stb} in Section~\ref{sec:pf-thm2}.
The result of Proposition~\ref{prop:cyc_swt} to follow yields that $\rho_\tau(\mathrm{M}_{K,\tau}(\bZ))\lesssim \overline{\Psi}_{\tau}(\bZ) \lesssim \rho_\tau(\bZ)$, so ignoring absolute constants,  we see that Theorem~\ref{thm:cyc_stb} provides a stronger guarantee than its counterpart Theorem~\ref{thm:cyc}. On the other hand, Theorem~\ref{thm:cyc_stb} relies on a stronger stability assumption than the one made in Theorem~\ref{thm:cyc}. In particular, we now require that the $(\score, \Alg)$ is leave-one-block-out stable not just with respect to the given sequence $\bZ$ but also with respect to its masked counterparts $\mathrm{M}_{K, \tau}(\bZ)$ and $\mathrm{M}_{K, \tau+1}(\bZ)$. 
This stability condition can be ensured if the $(\score, \Alg)$ pair is OOS stable to leaving \emph{two} random blocks---of sizes $\tau$ and $\tau+1$---out of the training data, as opposed to a single block (as required by Theorem~\ref{thm:cyc}).

\subsubsection{Relating cyclic coefficient under masking to mixing}

Since our eventual goal is to show a coverage result for mixing sequences, we now relate our cyclic embedding coefficient $\rho_\tau(\mathrm{M}_{K}(\bZ))$ to the switch coefficients from Definition~\ref{def:switch-coeff}. 

\begin{proposition}
\label{prop:cyc_swt}
We have 
\begin{subequations}
\begin{align} \label{eq:switch-embedding-1}
\overline{\Psi}_{\tau}(\bZ) \leq 2\rho_{\tau}(\bZ).
\end{align}
Moreover, for $K \sim \mathsf{Unif}(\{-\tau,\ldots,n\})$ drawn independently of the data, we have
\begin{align} \label{eq:switch-embedding-2}
\rho_{\tau}(\mathrm{M}_{K, \tau}(\bZ)) \leq \overline{\Psi}_\tau(\bZ)+\Psi_{0,\tau}(\bZ).
\end{align}
\end{subequations}
\end{proposition}
We prove Proposition~\ref{prop:cyc_swt} in Section~\ref{sec:pf-prop3}. Note that if $\bZ$ is stationary, then $\Psi_{0,\tau}(\bZ) = 0$. Moreover, the results of~\cite{barber2025predictive} show that in this case\footnote{When $\bZ$ is stationary, note that our $\overline{\Psi}(\bZ)$ is equal to that of~\citet{barber2025predictive} up to a multiplicative normalization factor.} $\frac{n +\tau + 1}{n} \cdot \overline{\Psi}_{\tau}(\bZ) \leq 2\beta(\tau)$, with no dependence on the conditional mixing coefficient.
Combining these pieces then immediately yields the following corollary when $\bZ$ is stationary and mixing.
\begin{corollary} \label{cor:switch}
Suppose $\Alg$ is symmetric and $K\sim\mathsf{Unif}(\{-\tau,\ldots,n\})$ is drawn independently of the data. Also suppose that the pair $(\score, \Alg)$ is OOS stable with respect to $\bZ$ with parameters $(\nu, t, \tau)$, OOS stable with respect to $\mathrm{M}_{K,\tau}(\bZ)$ with parameters $(\nu, t, \tau+1)$, and OOS stable with respect to $\mathrm{M}_{K,\tau+1}(\bZ)$ with parameters $(\nu, t, \tau)$. If $\bZ$ is stationary and $\beta$-mixing, then
\begin{align*}
\mathbb{P}\left(Y_{n+1}\in C^{\jk,3t}_{\level}(\bX_{n+1})\right) \geq 1-\alpha-5\sqrt{\nu+\frac{2\tau+2}{n}+ 2 \beta(\tau)}.  
\end{align*}
\end{corollary}
Corollary~\ref{cor:switch} should be viewed as complementary to Corollary~\ref{cor:cyc}. While Corollary~\ref{cor:cyc} is useful for mixing sequences that are also approximately Markovian (having small \(\beta^*(\tau)\) and $\beta(\tau)$ for an appropriately small $\tau$), Corollary~\ref{cor:switch} does not require any Markovian structure---the \(\beta\)-mixing condition alone suffices. However, we now need an additional stability assumption with respect to the masked sequences \(\mathrm{M}_{K,\tau}(\bZ)\) and \(\mathrm{M}_{K,\tau+1}(\bZ)\). 

\section{Experiments}
\label{sec:experiments}
In this section, we compare LWO with split CP (which, as discussed before, offers robustness to temporal dependence but may lead to larger prediction sets), and the vanilla jackknife (which may lead to undercoverage in a time series setting) on both synthetic and real data. These two baselines serve to illustrate the main trade-off between coverage and interval length, and our goal is to see if LWO navigates this trade-off favorably.\footnote{Code for our experiments is available at \url{https://github.com/scotthanyang/LWO}.} Additional experiments (including comparisons with a broader class of predictive inference methods) can be found in Appendix~\ref{sec:add_expt}.

\subsection{Simulated data}

We use our simulation setup from Section~\ref{sec:motivating}, which involves multidimensional responses. We also include a simulation experiment with univariate responses in Appendix~\ref{App:add_exp}.

\subsubsection{Data generation}

    As in the setup of Section~\ref{sec:motivating}, we generate our covariates as a sequence of random vectors \(X_t \in \mathbb{R}^d\) according to
    \[
    X_t = \omega_t + \omega_{t-1},
    \qquad
    \omega_t \stackrel{\mathrm{i.i.d.}}{\sim} \mathcal{N}(0,I_d),
    \]
    in dimension \(d=50\). The response at time $t$ is given by
$
    Y_t = X_{t+1}.
$
We consider the memoryless setting with $\bX_t = X_t$.
The prediction problem is thus to use the current observation \(X_t\) to forecast the next vector \(X_{t+1}\).

Our MA(1) process $\bZ$ exhibits mild dependence. It can be verified that $\rho_\tau(\bZ) = 0$ for all $\tau \geq 2$, and its \(\beta\)-mixing coefficients satisfy
$\beta(\tau)=0$ for all $\tau\ge 2$.

\subsubsection{Predictors and baseline methods}
\label{predictors}
We compare five different predictors. None of the predictors has an internal recurrent or sequential state, and the only memory (if any) enters through the covariates. 
\begin{itemize}
    \item \textbf{Ridge:} Linear regression with \(\ell^2_2\)-regularization parameter \(\lambda=1.0\). As discussed in Eq.~\eqref{eq:ridge-stable}, this predictor satisfies OOS stability in the sense of Definition~\ref{def:stab}.
    \item \textbf{KNN:} \(k\)-nearest-neighbor regression with \(k=10\). As discussed in Eq.~\eqref{eq:NN-stable}, this predictor is also OOS stable in the sense of Definition~\ref{def:stab}.
    \item \textbf{Random Forest (RF):} Random forest with \(10\) trees. Each tree has maximum depth \(5\) and minimum leaf size \(2\). While individual decision trees can be unstable under small changes to the training set, random forests as an ensemble method improves stability.
    \item \textbf{Kernel Regression (KR):} Gaussian kernel regression with bandwidth \(0.5\). KR ought to be stable when the kernel weights are sufficiently spread out.
    \item \textbf{MLP:} A multilayer perceptron with one hidden layer of width \(20\) and ReLU activations. We include MLP as a flexible nonlinear predictor, but this predictor is not verifiably stable in the sense of Definition~\ref{def:stab}.

\end{itemize}

We use the target coverage level \(1-\alpha=0.9\) and set \(n=200\). We compare the following predictive inference procedures:
\begin{itemize}
    \item \textbf{Split CP:} standard split conformal prediction. We split the sample evenly into a training set (50\%) and a calibration set (50\%).
    \item \textbf{Jackknife:} the vanilla leave-one-out procedure.
    \item \textbf{Leave-a-Window-Out (LWO):} our proposed method with window size \(\tau=5\).
\end{itemize}

\subsubsection{Evaluation}
We report empirical coverage and mean prediction-set size in each setting, averaged over \(500\) independent trials. Prediction set size is measured by the average radius of the Euclidean prediction ball. In the figures, bars show trial averages and error bars indicate \(\pm 1\) standard error over replications.

\begin{figure}
    \centering
    \includegraphics[width=0.8\linewidth]{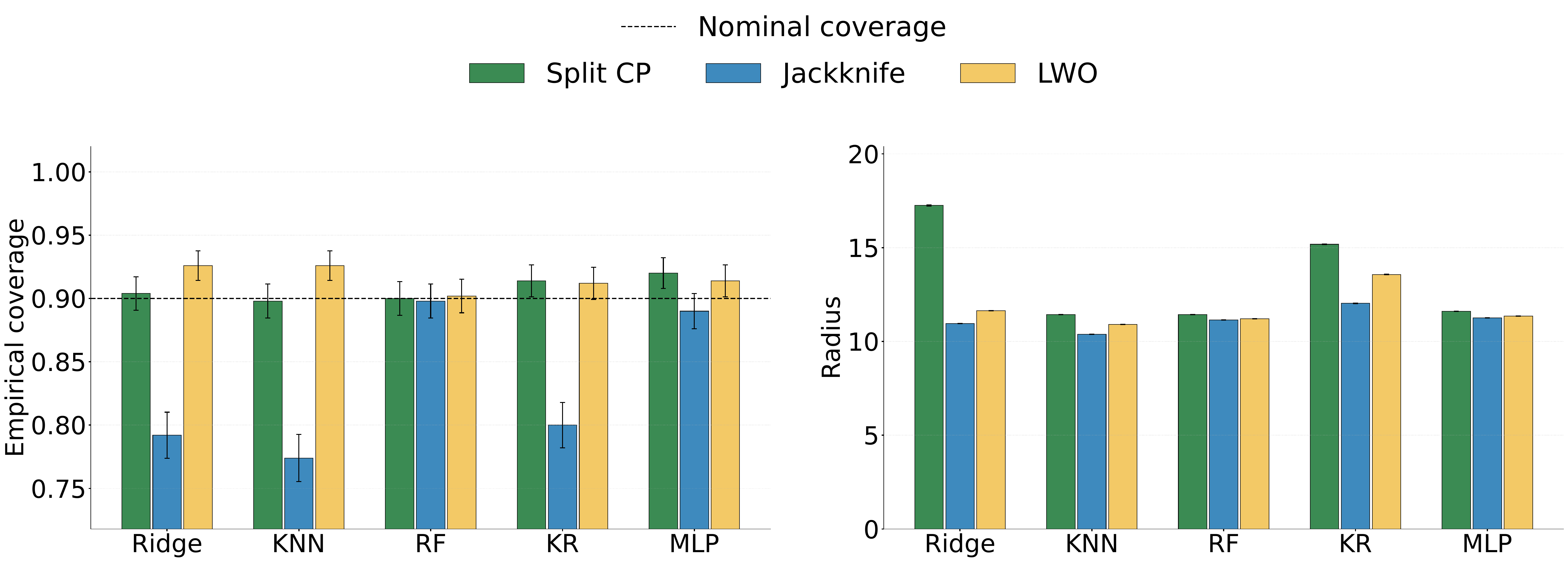}
\caption{Empirical performance on the multidimensional MA(1) process across five base predictors. 
(Left) Empirical coverage of Split CP, Jackknife, and LWO, with the dashed horizontal line marking the nominal \(90\%\) coverage level. 
(Right) Average prediction region radius for the same methods and predictors. 
Bars show means across $500$ repeated trials, and error bars indicate \(\pm 1\) standard error.}
    \label{fig:ma1_results}
\end{figure}

Figure~\ref{fig:ma1_results} shows that LWO is well calibrated on the MA(1) process, with empirical coverage close to the nominal \(0.9\) level across all five predictors. By contrast, the vanilla jackknife undercovers for several predictors, especially Ridge, KNN, and KR. Split CP exhibits close-to-nominal coverage for all five predictors, but returns significantly larger prediction sets than the other methods. Overall, Figure~\ref{fig:ma1_results} corroborates our theory: LWO preserves near-nominal coverage (comparable to split CP) while improving efficiency.

To develop some intuition for the coverage of Jackknife and LWO in this example, note that in the vanilla jackknife, there is an asymmetry between the test point and the leave-one-out proxy test points. At test time, predicting \(Y_{n+1}=X_{n+2}=\omega_{n+1}+\omega_{n+2}\) is intrinsically ``one-sided": the training data contains information about \(\omega_{n+1}\), but does not contain information about the new innovation \(\omega_{n+2}\). In contrast, when an interior point \(t\) is used as a proxy test point in the vanilla jackknife, the leave-one-out training set still contains observations on both sides of \(t\), carrying information about both $\omega_{t}$ and $\omega_{t + 1}$ when predicting \(Y_t=X_{t+1}=\omega_t+\omega_{t+1}\). Thus, the jackknife calibrates on easier, ``two-sided" prediction problems (leading to small prediction regions) and so undercovers at test time.
LWO breaks this asymmetry by removing the local future window around each proxy test point, preventing the calibration score from using observations that reveal the future innovation. As a result, its calibration residuals better mimic the information structure of the residual at the final test point. 

\subsection{Real data}
In this section, we evaluate LWO on two publicly available real time-series datasets.

\subsubsection{Datasets}

We consider the following two benchmark datasets.
\begin{itemize}
    \item \textbf{Traffic.} This dataset contains hourly road occupancy rates (values between \(0\) and \(1\)) collected from freeway sensors in the San Francisco Bay Area during 2015--2016.

    \item \textbf{Solar Energy.} This dataset contains solar power production measurements sampled every 10 minutes from \(137\) photovoltaic plants in Alabama during 2006.
\end{itemize}

On both datasets, we construct a univariate forecasting task by choosing a single coordinate of the multivariate data at each time index. Let \(W_t\) denote the resulting scalar time series. We consider one-step-ahead prediction with lag length \(L=24\), so that the covariate at time \(t\) is
\[
\bX_t = (W_{t-L+1},\ldots,W_t),
\]
and the response is
\[
Y_t = W_{t+1}.
\]
Thus, each predictor takes the most recent \(24\) observations as input and predicts the next one.

\begin{figure}[t]
\centering

\includegraphics[width=0.8\linewidth]{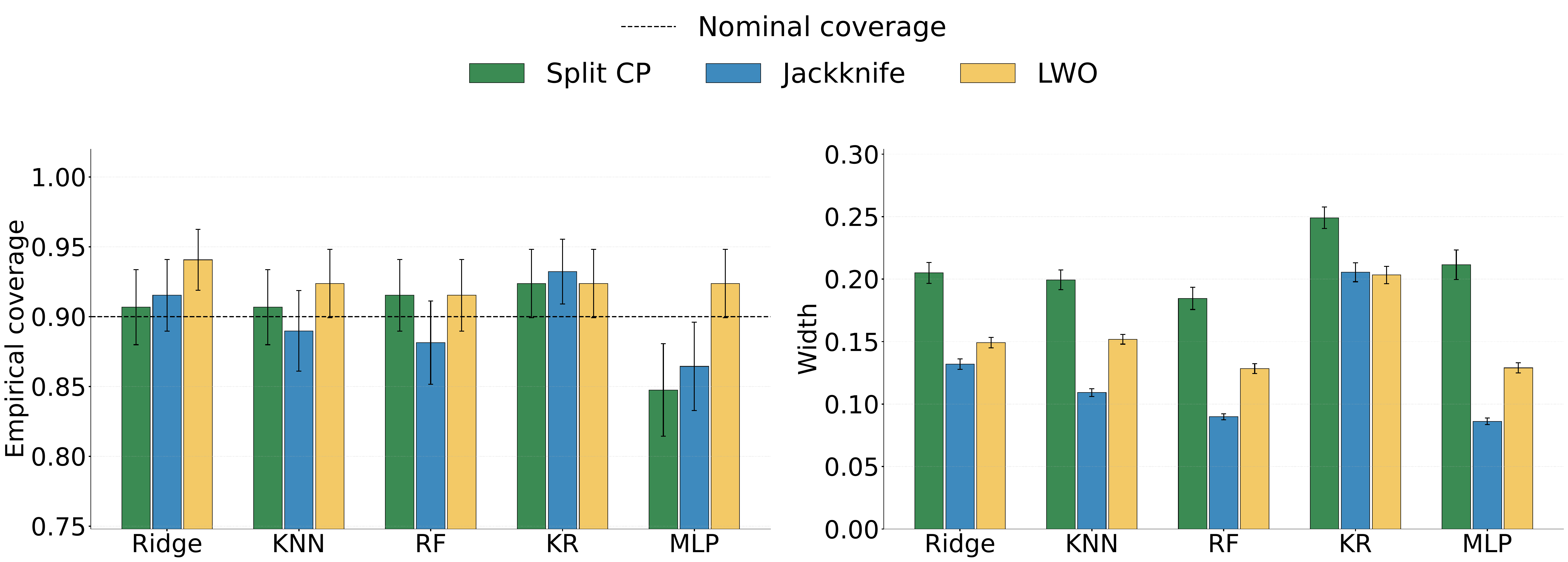}

\vspace{0.8em}

\includegraphics[width=0.8\linewidth]{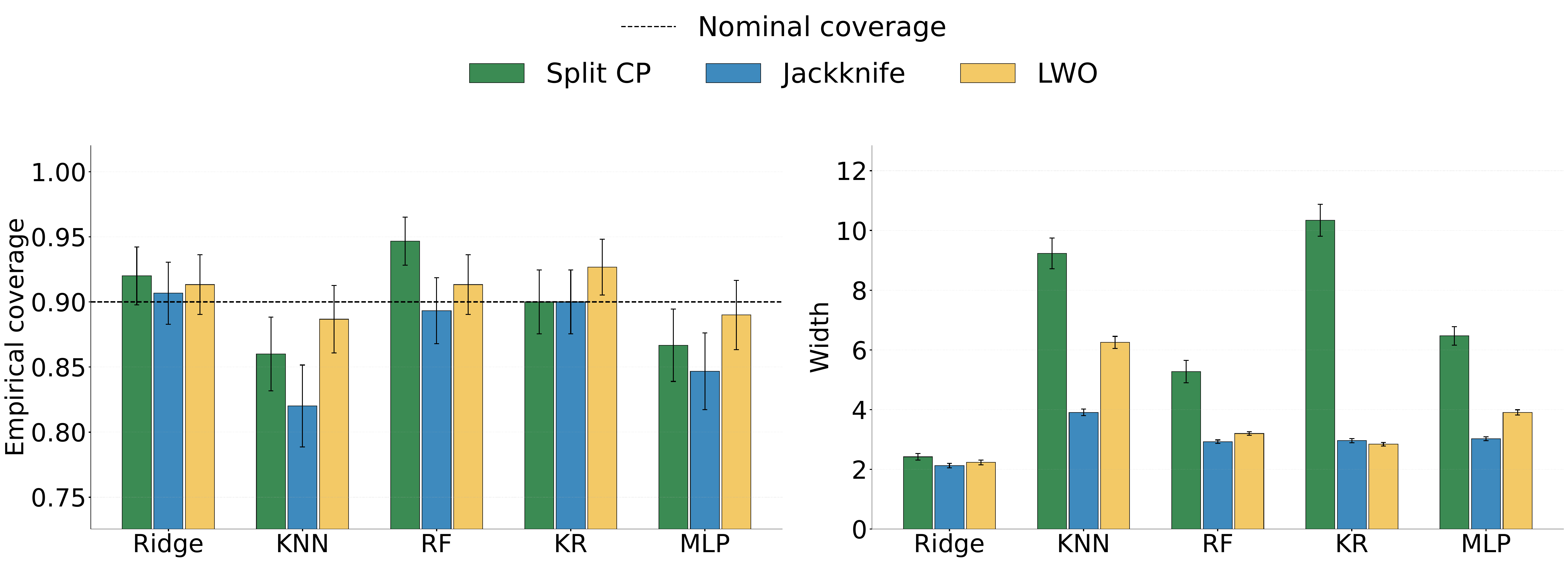}

\caption{Empirical performance on two real data benchmarks. (Top) Traffic dataset. (Bottom) Solar Energy dataset. In each row, the left panel reports empirical coverage, and the right panel reports average prediction-set size across base predictors. Bars show means across repeated trials, and error bars indicate \(\pm 1\) standard error. The dashed horizontal line marks nominal \(90\%\) coverage.}
\label{fig:real_data_results}
\end{figure}

To obtain multiple replicates of data from a single long series, we partition each dataset into disjoint local trajectories (``chunks''). Each chunk provides one forecasting problem of size \(n\), and consecutive chunks are separated by a gap of \(48\) time steps. We use $n=100$ for the Traffic dataset and $n=300$ for the Solar Energy dataset.
For both datasets, we fix the LWO window length at \(\tau=20\). These choices (\(L=24\), \(\tau=20\), and chunk gap \(48\)) are held fixed across all predictors and datasets, and we do not tune them separately for different methods.

\subsubsection{Predictors and baseline methods}
We compare the same five predictors (with the same hyperparameters) and three baseline methods described in Section~\ref{predictors}. As described above, for our LWO method, we use a window size $\tau=20$ for both real datasets.

\subsubsection{Evaluation} 

The real-data experiments in Figure~\ref{fig:real_data_results} reveal that Jackknife can still undercover when temporal dependence interacts unfavorably with the predictor. This failure is most pronounced on the Solar Energy dataset with KNN, where Jackknife falls well below nominal coverage, while LWO recovers coverage close to the target. A milder version of the same phenomenon appears on the Traffic dataset, where Jackknife is generally competitive but can undercover for predictors such as DT. In contrast, LWO is more consistent across predictors, typically maintaining near-nominal coverage.

Figure~\ref{fig:real_data_results} also shows the expected efficiency trade-off. Split CP is often well calibrated, but it tends to produce wider intervals because it trains on only part of the data. LWO avoids much of this cost of data splitting while still guarding against the local-dependence failures that affect leave-one-out methods. Overall, these real-data results suggest that LWO transfers beyond the synthetic examples: it is generally more robust than Jackknife under dependence, and often more efficient than Split CP.

\section{Proofs of main results} \label{sec:proofs}

In this section, we present proofs of all our main results. Before proceeding, we extend the definition of the masking operation $\mathrm{M}_{k,\tau}(\mathbf{w})$ applied to a vector $\mathbf{w}$ of length $m$ to indices \(k \notin \{-\tau,\ldots,m-1\}\), by interpreting such indices modulo \(m+\tau\). More precisely, if \(k \notin \{-\tau,\ldots,m-1\}\) then define \(k' := (k + \tau) \bmod (m+\tau)-\tau\), and define the corresponding masking operation as 
\begin{align} \label{eq:masking-convention}
\mathrm{M}_{k,\tau}(\mathbf{w}) := \mathrm{M}_{k',\tau}(\mathbf{w}),
\end{align}
where the RHS should be interpreted according to Definition~\ref{def:mask}, since \(k' \in \{-\tau,\ldots,m-1\}\). 

\subsection{Proof of Proposition~\ref{prop:cyc}} \label{sec:pf-prop1}

Let $\bZtil \in \mathcal{S}(n + \tau + 1)$ be the cyclically exchangeable sequence into which $\bZ$ embeds, so that $\bZ \overset{\textnormal{d}}{=} \bZtil_{1:(n+1)}$. 
Define the vector of ``circular scores" $s^{\circlearrowright} \in \mathbb{R}^{n + \tau + 1}$, as follows. For $k\in[n+\tau+1]$, write
    \[s^{\circlearrowright}_k = \score(\widetilde{Y}_k,\Alg[\widetilde{Z}_{k+\tau+1},\dots,\widetilde{Z}_{k-1}](\widetilde{\bX}_k))\]
    with the convention that the index that follows $n+\tau+1$ is $1$. In words, one should think of first placing the process $\bZtil$ on a circle of length $n + \tau + 1$. To evaluate the $k$-th score, we leave out the \emph{$\tau$ points appearing after index $k$ on the circle} and use the remaining points to train the predictor. See Figure~\ref{fig:score-circle}(left) for an illustration.
    Note that by construction, we have \[
s^{\circlearrowright}_{n+1} = \score(\widetilde{Y}_{n+1},\widehat{f}(\widetilde{\bX}_{n+1})) \overset{\textnormal{d}}{=} \score(Y_{n+1},\widehat{f}(\bX_{n+1})),
\]
where equality in distribution holds since $\bZ \overset{\textnormal{d}}{=} \bZtil_{1:(n + 1)}$.

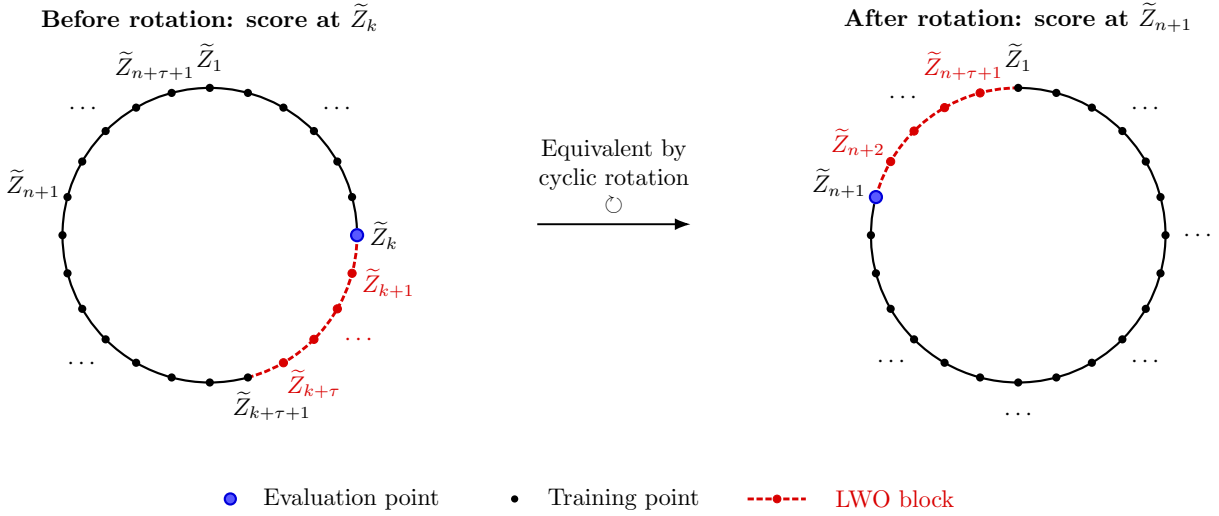
\begin{figure}[t]
\centering
\resizebox{0.98\linewidth}{!}{%
\begin{tikzpicture}[>=Latex,font=\small]
\def\R{2.15cm}
\def\Rlab{2.72cm}

\tikzset{
  basept/.style={black,fill=black},
  maskpt/.style={draw=red!85!black,fill=red!85!black},
  evalpt/.style={draw=blue!85!black,fill=blue!65!white,line width=0.9pt},
  circarc/.style={draw=black,line width=0.9pt},
  localarc/.style={draw=red!85!black,line width=1.2pt,dash pattern=on 2.6pt off 1.1pt}
}

\begin{scope}[xshift=-5.9cm]
    \node[font=\small\bfseries] at (0,3.18) {Before rotation: score at \(\widetilde Z_k\)};

    \draw[circarc] (180:\R) arc[start angle=180,end angle=0,radius=\R];
    \draw[circarc] (-75:\R) arc[start angle=-75,end angle=-180,radius=\R];

    \draw[localarc] (0:\R) arc[start angle=0,end angle=-75,radius=\R];

    \foreach \a in {180,165,150,135,120,105,90,75,60,45,30,15,-75,-90,-105,-120,-135,-150,-165} {
        \fill[basept] (\a:\R) circle (1.7pt);
    }

    \foreach \a in {-15,-30,-45,-60} {
        \fill[maskpt] (\a:\R) circle (1.7pt);
    }

    \node[anchor=east] at (160:2.18cm) {\(\widetilde Z_{n+1}\)};
    \node at (135:2.62cm) {\(\cdots\)};
    \node at (108:2.6cm) {\(\widetilde Z_{n+\tau+1}\)};
    \node at (90:2.6cm) {\(\widetilde Z_1\)};
    \node at (45:2.62cm) {\(\cdots\)};
    \node[anchor=west] at (0:2.2cm) {\(\widetilde Z_k\)};
    \node[text=red!85!black] at (-15:2.7cm) {\(\widetilde Z_{k+1}\)};
    \node[text=red!85!black] at (-35:2.65cm) {\(\cdots\)};
    \node[text=red!85!black] at (-55:2.65cm) {\(\widetilde Z_{k+\tau}\)};
    \node at (-70:2.65cm) {\(\widetilde Z_{k+\tau+1}\)};
    \node at (-135:2.65cm) {\(\cdots\)};

    \fill[white] (0:\R) circle (3.0pt);
    \filldraw[evalpt] (0:\R) circle (2.5pt);
\end{scope}

\draw[-Latex,thick] (-1.12,0.16) -- (1.12,0.16);
\node[align=center,font=\small] at (0,0.86) {Equivalent by\\ cyclic rotation\\[-0.1em]\(\circlearrowright\)};

\begin{scope}[xshift=5.9cm]
    \node[font=\small\bfseries] at (0,3.18) {After rotation: score at \(\widetilde Z_{n+1}\)};

    \draw[circarc] (90:\R) arc[start angle=90,end angle=-195,radius=\R];

    \draw[localarc] (165:\R) arc[start angle=165,end angle=90,radius=\R];

    \foreach \a in {90,75,60,45,30,15,0,-15,-30,-45,-60,-75,-90,-105,-120,-135,-150,-165,-180} {
        \fill[basept] (\a:\R) circle (1.7pt);
    }

    \foreach \a in {150,135,120,105} {
        \fill[maskpt] (\a:\R) circle (1.7pt);
    }

    \node[anchor=east] at (160:2.22cm) {\(\widetilde Z_{n+1}\)};
    \node[text=red!85!black] at (150:2.7cm) {\(\widetilde Z_{n+2}\)};
    \node at (130:2.62cm) {\(\cdots\)};
    \node[text=red!85!black] at (108:2.6cm) {\(\widetilde Z_{n+\tau+1}\)};
    \node at (90:2.6cm) {\(\widetilde Z_1\)};
    \node at (45:2.62cm) {\(\cdots\)};
    \node at (0:2.62cm) {\(\cdots\)};
    \node at (-45:2.62cm) {\(\cdots\)};
    \node at (-90:2.62cm) {\(\cdots\)};
    \node at (-135:2.62cm) {\(\cdots\)};

    \fill[white] (165:\R) circle (3.0pt);
    \filldraw[evalpt] (165:\R) circle (2.5pt);
\end{scope}

\begin{scope}[yshift=-3.85cm]
    \fill[white] (-5.6,0) circle (3.0pt);
    \filldraw[evalpt] (-5.6,0) circle (2.3pt);
    \node[anchor=west,font=\small] at (-5.25,0) {Evaluation point};

    \fill[basept] (-1.45,0) circle (1.5pt);
    \node[anchor=west,font=\small] at (-1.10,0) {Training point};

    \draw[localarc] (1.95,0) -- +(0.85,0);
    \fill[maskpt] (2.38,0) circle (1.6pt);
    \node[anchor=west,font=\small,red!85!black] at (3.10,0) {LWO block};
\end{scope}

\end{tikzpicture}%
}
\caption{
Illustration of the circular embedding used to compute the scores \(s^{\circlearrowright}\).
We place the cyclically exchangeable sequence
\(\widetilde{\bZ}=(\widetilde Z_1,\ldots,\widetilde Z_{n+\tau+1})\)
on a circle by gluing the end of the sequence back to its beginning, so that
\(\widetilde Z_{n+\tau+1}\) is followed by \(\widetilde Z_1\).
Black points indicate observations used for training, the blue point is the evaluation point, and the omitted LWO block is indicated by a red dashed arc together with red points marking the specific omitted observations.
(Left) To compute \(s^{\circlearrowright}_k\), the score is evaluated at \(\widetilde Z_k\), and the following \(\tau\) points on the circle are omitted.
(Right) By cyclic exchangeability of \(\widetilde{\bZ}\), this is equivalent to evaluating the score at \(\widetilde Z_{n+1}\), with the LWO block shifted to
\(\widetilde Z_{n+2},\ldots,\widetilde Z_{n+\tau+1}\).
}
\label{fig:score-circle}
\end{figure}

    By cyclic exchangeability of $\bZtil$ and by symmetry of $\Alg$, the scores $s^{\circlearrowright}_1,\dots,s^{\circlearrowright}_{n+\tau+1}$ are exchangeable, and consequently,
    \[\mathbb{P}\left(s^{\circlearrowright}_{n+1}\leq \Quantile_{1 - \alpha'} (s^{\circlearrowright}_1, \ldots,s^{\circlearrowright}_{n+\tau+1})\right)\geq 1-\alpha',\]
    for any choice of $\alpha'$. From Lemma~\ref{lem:quantile-subset} in Appendix~\ref{app:ancillary}, it holds that
    \[\Quantile_{1 - \alpha'} (s^{\circlearrowright}_1, \ldots, s^{\circlearrowright}_{n+\tau+1}) \leq \Quantile_{(1-\alpha')\left(1+\frac{\tau+1}{n}\right)} (s^{\circlearrowright}_1, \ldots,s^{\circlearrowright}_{n}),\]
    so that
     \begin{align} \label{eq:cyclic-step}
     \mathbb{P}\left(s^{\circlearrowright}_{n+1}\leq \Quantile_{(1-\alpha')\left(1+\frac{\tau+1}{n}\right)} (s^{\circlearrowright}_1, \ldots, s^{\circlearrowright}_{n})\right)\geq 1-\alpha'.
     \end{align}

It remains to relate the quantile of the scores $s^{\circlearrowright}_1, \ldots, s^{\circlearrowright}_{n}$ to the quantile of the scores $s_1, \ldots, s_{n}$ computed by the LWO algorithm when applied to $\bZ$. Since $\bZ \overset{\textnormal{d}}{=} \bZtil_{1:(n+1)}$, we can think of $s_k$ computed with $\Alg$ applied to the masked sequence $\mathrm{M}_{k-1,\tau+1}(\widetilde{Z}_1,\ldots,\widetilde{Z}_n)$.
On the other hand, the scores $s^{\circlearrowright}_k$ are computed with $\Alg$ applied to $(\widetilde{Z}_{k+\tau+1},\dots, \widetilde{Z}_{n + \tau + 1}, \widetilde{Z}_1, \ldots, \widetilde{Z}_{k-1})$. By cyclic exchangeability and symmetry, we therefore have that for all $k \in [n]$,
\begin{align} \label{eq:score-dist-equal}
&|s_k-s^{\circlearrowright}_k| \notag \\
&\overset{\textnormal{d}}{=}
\Bigl|\score\bigl(\widetilde{Y}_{k},\Alg[\mathrm{M}_{k-1,\tau+1}(\widetilde{Z}_1,\ldots,\widetilde{Z}_n)](\widetilde{\bX}_{k})\bigr) \notag \\
& \qquad - \score\bigl(\widetilde{Y}_{k},\Alg[\widetilde{Z}_1,\ldots,\widetilde{Z}_{k-1},\widetilde{Z}_{k+\tau+1},\ldots,\widetilde{Z}_{n+\tau+1}](\widetilde{\bX}_{k})\bigr)\Bigr| \notag \\
&\overset{\textnormal{d}}{=}\;
\Bigl|\score\bigl(\widetilde{Y}_{n+1},\Alg[\mathrm{M}_{n-k-\tau,\tau+1}(\widetilde{Z}_1,\ldots,\widetilde{Z}_n)](\widetilde{\bX}_{n+1})\bigr)
 - \score\bigl(\widetilde{Y}_{n+1},\Alg[\widetilde{Z}_1,\ldots,\widetilde{Z}_n](\widetilde{\bX}_{n+1})\bigr)\Bigr| \\
&\overset{\textnormal{d}}{=}\;
\Bigl|\score\bigl(Y_{n+1},\Alg[\mathrm{M}_{n-k-\tau,\tau+1}(Z_1,\ldots,Z_n)](\bX_{n+1})\bigr)
 - \score\bigl(Y_{n+1},\Alg[Z_1,\ldots,Z_n](\bX_{n+1})\bigr)\Bigr|. \label{eq:last-dist}
\end{align}
Eq.~\eqref{eq:last-dist} holds because $\bZ \overset{\textnormal{d}}{=} \bZtil_{1:(n + 1)}$; let us explain the relation~\eqref{eq:score-dist-equal} 
in more detail. Note that we have 
\[\score\bigl(\widetilde{Y}_{n+1},\Alg[\widetilde{Z}_1,\ldots,\widetilde{Z}_n](\widetilde{\bX}_{n+1})\bigr) \overset{\mathrm{d}}{=} \score\bigl(\widetilde{Y}_{k},\Alg[\widetilde{Z}_1,\ldots,\widetilde{Z}_{k-1},\widetilde{Z}_{k+\tau+1},\ldots,\widetilde{Z}_{n+\tau+1}](\widetilde{\bX}_{k})\bigr)
\]
as illustrated in Figure~\ref{fig:score-circle}(right), and equality in distribution of the first term in~\eqref{eq:score-dist-equal} follows by a similar argument (where there are two red blocks left out instead of just one).

Now by assumption, the $(\score, \Alg)$ pair satisfies OOS stability with respect to $\bZ$ with parameters $(\nu,t,\tau+1)$. Thus, the RHS of Eq.~\eqref{eq:last-dist} averaged over $k \in [n]$ is bounded above by $\nu$. Consequently, we have 
\begin{align} \label{eq:score-avg-prop1}
\frac{1}{n}\sum_{k=1}^{n}\mathbb{P}\bigl(|s^{\circlearrowright}_k-s_{k}|>t\bigr)\le \nu.
\end{align}
Applying Lemma~\ref{lem:quantile-set} from Appendix~\ref{app:ancillary} then yields
\begin{align*}
\mathbb{P}\left(\Quantile_{(1-\alpha')\left(1+\frac{\tau+1}{n}\right)} (s^{\circlearrowright}_1, \ldots, s^{\circlearrowright}_{n}) \leq \Quantile_{1-\alpha} (s_1, \ldots, s_{n})+t\right)\ge 1-\sqrt{\nu},
\end{align*}
if we define $\alpha'$ to satisfy
\[1-\alpha' = \frac{1-\alpha - \sqrt{\nu}}{1+\frac{\tau+1}{n}}.\]
Putting together the pieces, we see that the coverage is at least $\frac{1-\alpha - \sqrt{\nu}}{1+\frac{\tau+1}{n}} - \sqrt{\nu} \geq 1-\alpha - 2\sqrt{\nu} - \frac{\tau+1}{n}$, as claimed. 
\qed

\subsection{Proof of Theorem~\ref{thm:cyc}} \label{sec:pf-thm1}

First, we repeat the steps of the proof of Proposition~\ref{prop:cyc} for an arbitrary $\bZtil \in \mathcal{S}(n + \tau + 1)$, noting that Eq.~\eqref{eq:cyclic-step} still holds. Reproducing it below, we have
     \begin{align*}
     \mathbb{P}\left(s^{\circlearrowright}_{n+1}\leq \Quantile_{(1-\alpha')\left(1+\frac{\tau+1}{n}\right)} (s^{\circlearrowright}_1, \ldots, s^{\circlearrowright}_{n})\right)\geq 1-\alpha'.
     \end{align*}
     The difference from this point on is in how we relate the scores $s^{\circlearrowright}_1, \ldots, s^{\circlearrowright}_{n}$ to the scores $s_1, \ldots, s_{n}$ computed by the LWO algorithm when applied to $\bZ$. Define $\widetilde{s}_k$ as the LWO score computed using the predictor trained on $\mathrm{M}_{k-1,\tau+1}(\widetilde{Z}_1,\ldots,\widetilde{Z}_n)$. By cyclic exchangeability of $\bZtil$ and symmetry of $\Alg$, we have that for all $k \in [n]$,
\begin{align*}
|s^{\circlearrowright}_k-\widetilde{s}_k| \;&\overset{\textnormal{d}}{=}\;
\Bigl|\score\bigl(\widetilde{Y}_{n+1},\Alg[\mathrm{M}_{n-k-\tau,\tau+1}(\widetilde{Z}_1,\ldots,\widetilde{Z}_n)](\widetilde{\bX}_{n+1})\bigr)
 - \score\bigl(\widetilde{Y}_{n+1},\Alg[\widetilde{Z}_1,\ldots,\widetilde{Z}_n](\widetilde{\bX}_{n+1})\bigr)\Bigr|.
\end{align*}
Now, for each $k \in [n]$, we may bound the tail probability of the RHS above as
\begin{equation*}
\begin{aligned}
&\mathbb{P}\Bigg\{\Big|\score(\widetilde{Y}_{n + 1}, \Alg[\mathrm{M}_{n-k-\tau,\tau+1}(\widetilde{Z}_1,\ldots,\widetilde{Z}_n)](\widetilde{\bX}_{n + 1}))-\score(\widetilde{Y}_{n + 1}, \Alg[\widetilde{Z}_1,\ldots, \widetilde{Z}_{n}](\widetilde{\bX}_{n + 1}))\Big|\ge t\Bigg\}\\
&\le 
\mathbb{P}\Bigg\{\Big|\score(Y_{n + 1}, \Alg[\mathrm{M}_{n-k-\tau,\tau+1}(Z_1,\ldots,Z_n)](\bX_{n + 1}))-\score(Y_{n + 1}, \Alg[Z_1,\ldots, Z_{n}](\bX_{n + 1}))\Big|\ge t\Bigg\} \\
&\quad+\mathrm{d}_{\mathrm{TV}}(\bZ,\widetilde{\bZ}_{1:(n+1)}),
\end{aligned}
\end{equation*}
where we replace $\widetilde{\bZ}_{1:(n+1)}$ in the probability by $\bZ$ and pay an additive price of the TV distance between these two sequences. Furthermore, by assumption, the pair $(\score,\Alg)$ is OOS stable with respect to $\bZ$ with parameters $(\nu,t,\tau+1)$. 
Putting this together with the above display, we have
$
\frac{1}{n}\sum_{k=1}^{n}\mathbb{P}\bigl(|s^{\circlearrowright}_k-\widetilde{s}_{k}|>t\bigr)\le \nu + \mathrm{d}_{\mathrm{TV}}(\bZ,\widetilde{\bZ}_{1:(n+1)}).
$
Applying Lemma~\ref{lem:quantile-set} from Appendix~\ref{app:ancillary} then yields
\begin{align*}
\mathbb{P}\left(\Quantile_{(1-\alpha')\left(1+\frac{\tau+1}{n}\right)} (s^{\circlearrowright}_1, \ldots, s^{\circlearrowright}_{n}) \leq \Quantile_{1-\alpha} (\widetilde{s}_1, \ldots, \widetilde{s}_{n})+t\right)\ge 1-\sqrt{\nu + \mathrm{d}_{\mathrm{TV}}(\bZ,\widetilde{\bZ}_{1:(n+1)})},
\end{align*}
if we define $\alpha'$ to satisfy
\[1-\alpha' = \frac{1-\alpha - \sqrt{\nu + \mathrm{d}_{\mathrm{TV}}(\bZ,\widetilde{\bZ}_{1:(n+1)})}}{1+\frac{\tau+1}{n}}.\]
Furthermore, by definition, we have $s^{\circlearrowright}_{n+1} = \widetilde{s}_{n + 1}$.
Putting together the pieces then yields
\begin{align*}
     \mathbb{P}\left(\widetilde{s}_{n+1}\leq \Quantile_{1 - \alpha} (\widetilde{s}_1, \ldots, \widetilde{s}_{n}) + t\right) &\geq \frac{1-\alpha - \sqrt{\nu + \mathrm{d}_{\mathrm{TV}}(\bZ,\widetilde{\bZ}_{1:(n+1)})}}{1+\frac{\tau+1}{n}} - \sqrt{\nu + \mathrm{d}_{\mathrm{TV}}(\bZ,\widetilde{\bZ}_{1:(n+1)})} \\
     &\geq 1-\alpha - 2\sqrt{\nu + \mathrm{d}_{\mathrm{TV}}(\bZ,\widetilde{\bZ}_{1:(n+1)})} - \frac{\tau+1}{n}.
     \end{align*}
Since the LHS of the above expression can be expressed as the expectation of a bounded function of $\bZtil_{1:(n+1)}$, an analogous bound holds for the same function applied to $\bZ$ up to an additive term $\mathrm{d}_{\mathrm{TV}}(\bZ,\widetilde{\bZ}_{1:(n+1)})$. Concretely, we have
\begin{align*}
     \mathbb{P}\left(s_{n+1}\leq \Quantile_{1 - \alpha} (s_1, \ldots, s_{n}) + t\right) \geq 1-\alpha - 2\sqrt{\nu + \mathrm{d}_{\mathrm{TV}}(\bZ,\widetilde{\bZ}_{1:(n+1)})} - \frac{\tau+1}{n} - \mathrm{d}_{\mathrm{TV}}(\bZ,\widetilde{\bZ}_{1:(n+1)}).
     \end{align*}
Taking an infimum of the TV over $\bZtil \in \mathcal{S}(n + \tau + 1)$ then yields the claimed result.
\qed

\subsection{Proof of Proposition~\ref{prop:cycliccoef}} \label{sec:pf-prop2}

We will use notation such as \(\mathbb{P}_{(Z_1,\ldots,Z_\tau)}\) to denote marginal distributions of subvectors of \(\bZ\), and notation such as
\[
\mathbb{P}_{(Z_{\tau+1},\ldots,Z_{2\tau})\mid (Z_1,\ldots,Z_\tau,Z_{2\tau+1},\ldots,Z_{3\tau})}
\]
to denote conditional distributions of subvectors of \(\bZ\).

First we define a distribution that satisfies cyclic exchangeability. Sample
\[
(Z_{n+2},\ldots,Z_{n+\tau+1})\mid \bZ
\sim
\mathbb{P}_{(Z_{\tau+1},\ldots,Z_{2\tau})\mid (Z_1,\ldots,Z_\tau,Z_{2\tau+1},\ldots,Z_{3\tau})}
\bigl(\cdot\mid Z_{n+2-\tau},\ldots,Z_{n+1},Z_1,\ldots,Z_\tau\bigr).
\]
Next, draw \(K\sim\mathsf{Unif}([n+\tau+1])\) independently of \((Z_1,\ldots,Z_{n+\tau+1})\), and define \(\widetilde{\bZ}\) as a rotation of \((Z_1,\ldots,Z_{n+\tau+1})\) by \(K\) positions:
\[
\widetilde{\bZ}
=
(Z_{K},\ldots,Z_{n+\tau+1},Z_1,\ldots,Z_{K-1}).
\]
By construction, we have \(\widetilde{\bZ} \in \mathcal{S}(n + \tau + 1)\) and by definition of the cyclic exchangeability coefficient, we have
$
\rho_\tau(\bZ)
\le
\mathrm d_{\mathrm{TV}}\bigl(\bZ,(\widetilde Z_1,\ldots,\widetilde Z_{n+1})\bigr).
$

For each \(k\in[n+\tau+1]\), let \(\mathbb{P}^{\{k\}}\) denote the distribution of
\begin{align} \label{eq:Zk-def}
\bZ^{\{k\}}=(Z_{k},\ldots,Z_{n+\tau+1},Z_1,\ldots,Z_{k-1}).
\end{align}
Then the distribution of \(\widetilde{\bZ}\) can be written as the mixture
\[
\widetilde{\bZ}
\sim
\frac{1}{n+\tau+1}\sum_{k=1}^{n+\tau+1}\mathbb{P}^{\{k\}}.
\]
Therefore, since total variation distance is convex in each argument (in the space of measures), we have
\[
\mathrm d_{\mathrm{TV}}\bigl(\bZ,(\widetilde Z_1,\ldots,\widetilde Z_{n+1})\bigr)
\le
\frac{1}{n+\tau+1}
\sum_{k=1}^{n+\tau+1}
\mathrm d_{\mathrm{TV}}\bigl(\bZ,(Z^{\{k\}}_1,\ldots,Z^{\{k\}}_{n+1})\bigr).
\]
Consequently,
\begin{align*}
\rho_\tau(\bZ)
&\le
\frac{1}{n+\tau+1}
\sum_{k=1}^{n+\tau+1}
\mathrm d_{\mathrm{TV}}\bigl(\bZ,(Z^{\{k\}}_1,\ldots,Z^{\{k\}}_{n+1})\bigr) \\
&\overset{\1}{\le} \frac{4 \tau}{n+\tau+1}  + \frac{1}{n+\tau+1}
\sum_{k = 2\tau + 2}^{n-\tau+1}
\mathrm d_{\mathrm{TV}}\bigl(\bZ,(Z^{\{k\}}_1,\ldots,Z^{\{k\}}_{n+1})\bigr)\\
&\overset{\2}{\leq} \frac{4 \tau}{n+\tau+1}  + 2\beta(\tau) + 2 \beta^*(\tau).
\end{align*}
Here step $\1$ follows because for $k = 1$, we have $\bZ^{\{k\}}_{1:(n+1)} = \bZ$, and for $k \in \{2,\ldots,2\tau+1\}\cup\{n-\tau+2,\ldots,n+\tau+1\}$ we have bounded the total variation distance by $1$. On the other hand, step $\2$ follows from Lemma~\ref{lem:Zk-TV-bd} from Appendix~\ref{app:tech-lemma1}.
\qed

\subsection{Proof of Theorem~\ref{thm:cyc_stb}}  \label{sec:pf-thm2}

Throughout, suppose $K$ is drawn uniformly at random from the set $\{-\tau,\ldots,n\}$ independently of the data. For any $k,k' \in \{-\tau,\ldots,n\}$, we define the following shorthand for masked sequences: 
\begin{align*}
\bZ^{(k)} &= \mathrm{M}_{k,\tau}(\bZ),\\
\bZ^{(k,k')}&=\mathrm{M}_{k,\tau}\circ\mathrm{M}_{k',\tau}(\bZ).
\end{align*}
Also recall our convention~\eqref{eq:masking-convention} for masking $\mathrm{M}_{k, \tau}(\bZ)$ when $k \notin \{-\tau, \ldots, n\}$.
We also define $Y_k=\bX_k=\star$ for $k\notin [n+1]$, so that for such $k$, we have $\score(Y_k, f(\bX_k)) = 0$ for any function $f$.
We present the proof in two steps. We first apply Theorem~\ref{thm:cyc} to obtain a guarantee for LWO if applied to the masked sequence $\bZ^{(K)}$. Next, we relate the coverage event of LWO on the masked sequence $\bZ^{(K)}$ to the coverage event of LWO on the true sequence $\bZ$.

\paragraph{Step 1: Constructing coverage guarantee for algorithm applied on $\bZ^{(K)}$.} 
Let $\big\{ s_{k}^{(K)} \big\}_{k = 1}^{n + 1}$ denote the scores computed by LWO on $\bZ^{(K)}$. From the definition, the scores can be written as
\begin{align*}
s_{k}^{(K)}= \score(Y^{(K)}_{k},\Alg[\mathrm{M}_{k-1,\tau+1}(\bZ^{(K)}_{1:n})](\bX^{(K)}_k)).
\end{align*}

By assumption, the $(\score, \Alg)$ pair is OOS stable with respect to $\bZ^{(K)}$ with parameters $(\nu, t, \tau + 1)$. Thus, applying Theorem~\ref{thm:cyc} to $\bZ^{(K)}$, we have
\begin{equation*}
\begin{aligned}
\mathbb{P}\left(s_{n+1}^{(K)}\leq \Quantile_{1 - \alpha} (s_{1}^{(K)}, \ldots, s_{n}^{(K)}) + t\right) \geq 1-\alpha -2\sqrt{\nu+\rho_{\tau}(\bZ^{(K)})} - \frac{\tau+1}{n}-\rho_{\tau}(\bZ^{(K)}).
\end{aligned}
\end{equation*}

\paragraph{Step 2: Connecting coverage events on original and masked sequence.} 
For each $k \in \{-\tau,\ldots,n\}$, we begin by rewriting the LWO scores as follows:
\begin{align*}
s_{k}^{(K)}=\score(Y^{(K,k)}_{k},\Alg[\mathrm{M}_{k-1,\tau+1}(\bZ^{(K,k)}_{1:n})](\bX^{(K,k)}_{k})),
\end{align*}
and
\begin{align*}
s_{k} =\score(Y^{(k)}_{k},\Alg[\mathrm{M}_{k-1,\tau+1}(\bZ^{(k)}_{1:n})](\bX^{(k)}_{k})).
\end{align*}
Note that if $k\in\{-\tau,\ldots,0\}$, then we have $\score(Y_k, f(\bX_k)) = 0$ and so by our convention both scores defined above evaluate to zero.

In order to relate the coverage events on $\bZ$ and $\bZ^{(K)}$, we bound the two probabilities \newline \mbox{$\mathbb{P}\Big(|s_{n+1}-s_{n+1}^{(K)}|\geq t\Big)$} and $\frac{1}{n}\sum_{k=1}^{n}\mathbb{P}(|s_{k}-s^{(K)}_{k}|\ge t)$. Starting with the first term, we have
\begin{align}
&\mathbb{P}\Big(|s_{n+1}-s_{n+1}^{(K)}|\geq t\Big) \notag \\
&=\frac{1}{n+\tau+1}\Bigg[\sum_{k\in\{-\tau,\ldots,n-\tau-1\}}\mathbb{P}\Big(|s_{n+1}-s_{n+1}^{(k)}|\geq t\Big)+\sum_{k\in\{n-\tau,\ldots,n\}}\mathbb{P}\Big(|s_{n+1}-s_{n+1}^{(k)}|\geq t\Big)\Bigg] \notag \\
&\overset{\1}{\le} \frac{n \nu}{n+\tau+1} + \frac{\tau + 1}{n + \tau + 1} \notag \\
&\le \nu+\frac{\tau+1}{n}. \label{eq:last-point-related}
\end{align}
In step $\1$, we have used the fact that $(\score, \Alg)$ pair satisfies OOS stability with respect to $\bZ$ with parameters $(\nu, t, \tau)$ to bound the first term, and the fact that probabilities are bounded above by $1$ to control the second term.

We now turn to bounding the term $\frac{1}{n}\sum_{k=1}^{n}\mathbb{P}(|s_{k}-s^{(K)}_{k}|\ge t)$, which ought to be small when the predictor satisfies stability when evaluated at a random training point (as opposed to the test point). But Lemma~\ref{lem:strong_stab} in Appendix~\ref{sec:stability-transfer} relates these two notions of stability to each other using the switch coefficients. Applying it, we have
\begin{align*}
\frac{1}{n}\sum_{k=1}^{n}\mathbb{P}(|s_{k}-s^{(K)}_{k}|\ge t) 
\le\nu+\frac{\tau}{n+\tau+1}+\frac{n+\tau+1}{n} \cdot \overline\Psi_\tau(\bZ).
\end{align*}

Having bounded the two probabilities of interest, we now proceed to relating the coverage events of LWO on $\bZ^{(K)}$ and $\bZ$.
Applying Lemma~\ref{lem:quantile-set} from Appendix~\ref{app:ancillary}, we can relate the quantile of the LWO scores on $\bZ^{(K)}$ to the quantile of the LWO scores on $\bZ$ via
\begin{align}
&\mathbb{P}\Bigl(\Quantile_{1-\alpha-\sqrt{\nu+\frac{\tau}{n+\tau+1}+\frac{n+\tau+1}{n}\overline\Psi_\tau(\bZ)}}(s_{1}^{(K)},\ldots,s_n^{(K)})\le\Quantile_{1-\alpha}(s_1,\ldots,s_{n})+t\Bigr)\notag\\
&\quad\quad \ge 1-\sqrt{\nu+\frac{\tau}{n+\tau+1}+\frac{n+\tau+1}{n} \cdot \overline\Psi_\tau(\bZ)}.  \label{eq:quantile-related}
\end{align}

Thus,
the coverage event on $\bZ$ satisfies the chain of bounds
\begin{align*}
&\mathbb{P}\Big(s_{n+1}\leq \Quantile_{1 - \alpha} (s_{1}, \ldots, s_{n}) +3 t\Big) \\
&\qquad \qquad \overset{\1}{\ge} \mathbb{P}\Big(s_{n+1}^{(K)}\leq \Quantile_{1 - \alpha} (s_{1}, \ldots, s_{n}) +2 t\Big)-\nu-\frac{\tau+1}{n}\\
&\qquad \qquad \overset{\2}{\ge} \mathbb{P}\left(s_{n+1}^{(K)}\leq \Quantile_{1 - \alpha-\sqrt{\nu+\frac{\tau}{n+\tau+1}+\frac{n+\tau+1}{n} \cdot \overline\Psi_\tau(\bZ)}} (s_{1}^{(K)}, \ldots, s_{n}^{(K)}) + t\right)\\
&\qquad \qquad \qquad \qquad -\nu-\frac{\tau+1}{n}-\sqrt{\nu+\frac{\tau}{n+\tau+1}+\frac{n+\tau+1}{n} \cdot \overline\Psi_\tau(\bZ)},
\end{align*}
where step $\1$ follows from Eq.~\eqref{eq:last-point-related} and step $\2$ from Eq.~\eqref{eq:quantile-related}.
Using the definition of the quantile and performing some algebra, then yields
\begin{align*}
&\mathbb{P}\Big(s_{n+1}\leq \Quantile_{1 - \alpha} (s_{1}, \ldots, s_{n}) +3 t\Big) \\
&\qquad \qquad \ge 1-\alpha-\nu-2\sqrt{\nu+\frac{\tau}{n+\tau+1}+\frac{n+\tau+1}{n} \cdot \overline\Psi_\tau(\bZ)}-2\sqrt{\nu+\rho_{\tau}(\bZ^{(K)})}\\
&\qquad \qquad \qquad \qquad - \frac{2\tau+2}{n}-\rho_{\tau}(\bZ^{(K)})\\
&\qquad \qquad \ge 1-\alpha-3\sqrt{\nu+\frac{2\tau+2}{n}+\rho_{\tau}(\bZ^{(K)})}-2\sqrt{\nu+\frac{\tau}{n}+\frac{n+\tau+1}{n} \cdot \overline\Psi_\tau(\bZ)},
\end{align*}
as claimed.
\qed

\subsection{Proof of Proposition~\ref{prop:cyc_swt}} \label{sec:pf-prop3}
We prove the two claimed inequalities separately.

\subsubsection{Proof of Ineq.~\eqref{eq:switch-embedding-1}}

Fix an arbitrary cyclically exchangeable sequence $\widetilde{\bZ}\in \mathcal{S}(n+\tau+1)$.
For each \(k\in\{-\tau,\ldots,n\}\), we have
\begin{align*}
&\Psi_{k,\tau}(\bZ) \\
&\overset{\1}{\le}
\mathrm d_{\mathrm{TV}}\bigl(
    \Delta^0_{k,\tau}(\bZ),
    \Delta^0_{k,\tau}(\widetilde{\bZ}_{1:(n+1)})
\bigr) +
\mathrm d_{\mathrm{TV}}\bigl(
    \Delta^0_{k,\tau}(\widetilde{\bZ}_{1:(n+1)}),
    \Delta^1_{k,\tau}(\widetilde{\bZ}_{1:(n+1)})
\bigr)
+
\mathrm d_{\mathrm{TV}}\bigl(
    \Delta^1_{k,\tau}(\widetilde{\bZ}_{1:(n+1)}),
    \Delta^1_{k,\tau}(\bZ)
\bigr) \\
&\overset{\2}{\leq} 2 \mathrm{d}_{\mathrm{TV}}(\bZ, 
    \widetilde{\bZ}_{1:(n+1)}) + \mathrm d_{\mathrm{TV}}\bigl(
    \Delta^0_{k,\tau}(\widetilde{\bZ}_{1:(n+1)}),
    \Delta^1_{k,\tau}(\widetilde{\bZ}_{1:(n+1)})
\bigr) \\
&\overset{\3}{=} 2 \mathrm{d}_{\mathrm{TV}}(\bZ, 
    \widetilde{\bZ}_{1:(n+1)}).
\end{align*}
Here step $\1$ follows by triangle inequality and step $\2$ follows from the data processing inequality, since \(\Delta^0_{k,\tau}\) and \(\Delta^1_{k,\tau}\) are measurable maps of the length-\((n+1)\) sequence. Step $\3$ follows by cyclic exchangeability of \(\widetilde{\bZ}\), whereby
$\Delta^0_{k,\tau}(\widetilde{\bZ}_{1:(n+1)}) \overset{\textnormal{d}}{=}
    \Delta^1_{k,\tau}(\widetilde{\bZ}_{1:(n+1)})$.

Averaging the above inequality over all \(k\in\{-\tau,\ldots,n\}\) yields
\[
    \overline{\Psi}_{\tau}(\bZ)
    \le
    2\,
    \mathrm d_{\mathrm{TV}}\bigl(
        \bZ,\widetilde{\bZ}_{1:(n+1)}
    \bigr).
\]
Since this holds for every
\(\widetilde{\bZ}\in\mathcal S(n+\tau+1)\), taking the infimum over
\(\widetilde{\bZ}\) yields
$
    \overline{\Psi}_{\tau}(\bZ)
    \le
    2\rho_\tau(\bZ),
$
as claimed. \qed

\subsubsection{Proof of Ineq.~\eqref{eq:switch-embedding-2}}

We prove a stronger, two-sided version of this result.
 Define the switch coefficients between uniform mixtures of deleted data as 
\begin{align*}
\underline{\Psi}_{\tau}(\bZ):=\mathrm{d}_{\mathrm{TV}}\!\left(\Delta^{0}_{K, \tau}(\bZ), \Delta^1_{K, \tau}(\bZ) \right)
\end{align*}
for $K \sim \mathsf{Unif}(\{-\tau,\ldots,n\})$ drawn independently of everything else. We now prove the following two-sided bound for such $K \sim \mathsf{Unif}(\{-\tau,\ldots,n\})$:
\begin{align} \label{eq:two-sided-switch}
\frac{1}{2} \underline{\Psi}_{\tau}(\bZ) \overset{(a)}{\leq} \rho_{\tau}(\mathrm{M}_{K, \tau}(\bZ)) \overset{(b)}{\leq} \underline{\Psi}_\tau(\bZ)+\Psi_{0,\tau}(\bZ).
\end{align}
Notice that $\underline{\Psi}_{\tau}(\bZ)\le\overline{\Psi}_{\tau}(\bZ)$ due to the convexity of the TV distance (in the space of measures), so we recover Ineq.~\eqref{eq:switch-embedding-2} from Ineq.~\eqref{eq:two-sided-switch}, part (b).

\paragraph{Proof of Ineq.~\eqref{eq:two-sided-switch}, part (a):}
Let $\bZtil \in \mathcal{S}(n + \tau + 1)$ denote a cyclically exchangeable sequence. We will use the shorthand $\widetilde{\bZ}_{a:b}=(\widetilde{Z}_a,\ldots,\widetilde{Z}_{n+\tau+1},\widetilde{Z}_1,\ldots,\widetilde{Z}_b)$ if $a>b$.  Drawing $K \sim \mathsf{Unif}(\{-\tau,\ldots,n\})$ independently of everything else, we have
\begin{align*}
\underline{\Psi}_{\tau}(\mathbf{Z}) &= \mathrm{d}_{\mathrm{TV}}\!\left(\Delta_{K,\tau}^{0}(\bZ), \Delta_{K,\tau}^{1}(\bZ)\right) \\
&\le \mathrm{d}_{\mathrm{TV}}\!\left(\Delta_{K,\tau}^{0}(\bZ), \Delta_{K,\tau}^{0}(\widetilde{\bZ}_{1:(n+1)})\right)+ \mathrm{d}_{\mathrm{TV}}\!\left(\Delta_{K,\tau}^{0}(\widetilde{\bZ}_{1:(n+1)}), \Delta_{K,\tau}^{1}(\widetilde{\bZ}_{1:(n+1)})\right)\\ 
&\quad+ \mathrm{d}_{\mathrm{TV}}\!\left(\Delta_{K,\tau}^{1}(\widetilde{\bZ}_{1:(n+1)}),\Delta_{K,\tau}^{1}(\bZ)\right) \\
&\overset{\1}{\le} \mathrm{d}_{\mathrm{TV}}\!\left(\mathrm{M}_{K,\tau}(\bZ),\widetilde{\bZ}_{1:(n+1)} \right) + \mathrm{d}_{\mathrm{TV}}\!\left(\widetilde{\bZ}_{1:(n+1)},\widetilde{\bZ}_{(n+2-K):(n-K-\tau+1)}\right) + \mathrm{d}_{\mathrm{TV}}\!\left(\widetilde{\bZ}_{1:(n+1)},\mathrm{M}_{n-K-\tau+1,\tau}(\bZ)\right) \\
&\overset{\2}{\le} \mathrm{d}_{\mathrm{TV}}\!\left(\mathrm{M}_{K,\tau}(\bZ),\widetilde{\bZ}_{1:(n+1)} \right) + \mathrm{d}_{\mathrm{TV}}\!\left(\widetilde{\bZ}_{1:(n+1)},\mathrm{M}_{n-K-\tau+1,\tau}(\bZ)\right) \\
&\overset{\3}{=} 2\mathrm{d}_{\mathrm{TV}}\!\left(\mathrm{M}_{K,\tau}(\mathbf{Z}), \widetilde{\mathbf{Z}}_{1:(n+1)}\right).
\end{align*}
 Step $\1$ holds by the marginalization property of TV distance: because the deletion function $\Delta^0$ extracts subsets of the underlying sequences, the distance between these marginals is upper-bounded by the distance between the full sequences.
Step $\2$ holds because $\bZtil$ is cyclically exchangeable, and step $\3$ holds because $\mathrm{M}_{K,\tau}(\bZ)\overset{\textnormal{d}}{=}\mathrm{M}_{n-K-\tau+1,\tau}(\bZ)$ when $K\sim \mathsf{Unif}(\{-\tau,\ldots,n\})$, where we recall our convention~\eqref{eq:masking-convention} for masking. We may take an infimum of the RHS over $\bZtil$ to conclude.

\paragraph{Proof of Ineq.~\eqref{eq:two-sided-switch}, part (b):}
The sequence $\bZ$ and $\mathrm{M}_{0,\tau}(\bZ)$ can be divided into three blocks as follows:
\[
\bZ
=
(\underbrace{Z_1,\ldots,Z_{\tau}}_{\tau},
 \underbrace{Z_{\tau+1},\ldots,Z_{n-\tau+1}}_{n+1-2\tau},
 \underbrace{Z_{n-\tau+2},\ldots,Z_{n+1}}_{\tau}),
\]
and
\[
\mathrm{M}_{0,\tau}(\bZ)
=
(\underbrace{\star,\ldots,\star}_{\tau},
 \underbrace{Z_{\tau+1},\ldots,Z_{n-\tau+1}}_{n+1-2\tau},
 \underbrace{Z_{n-\tau+2},\ldots,Z_{n+1}}_{\tau}).
\]

We now construct a cyclically exchangeable sequence $\widetilde{\bZ}$ from $\mathrm{M}_{0,\tau}(\bZ)$. First, sample \[
Z_{n+2},\ldots,Z_{n+\tau+1} \; \mid \; \mathrm{M}_{0,\tau}(\bZ) \sim \mathbb{P}_{(Z_{n-\tau+2},\ldots,Z_{n+1}) \mid (Z_1,\ldots,Z_{n-\tau+1})} (\;\cdot\; | Z_{\tau+1},\ldots,Z_{n+1}).
\]
Then draw $K \sim \mathsf{Unif}([n+\tau+1])$ independently of $\bZ$ and define
    \[
    \widetilde{\bZ}
    =
    (Z_K,\ldots,Z_{n+\tau+1},\star,\ldots,\star,Z_{\tau+1},\ldots,Z_{K-1}).
    \]

By construction, we have $\bZtil \in \mathcal{S}(n + \tau + 1)$. Moreover, the conditional distribution of $Z_{n+2},\ldots,Z_{n+\tau+1}$ given $Z_{\tau+1},\ldots,Z_{n+1}$ is the same as the conditional distribution of $Z_{n-\tau+2},\ldots,Z_{n+1}$ given $Z_1,\ldots,Z_{n-\tau+1}$. Thus, the laws of $\bZ$ and $\bZ_{(\tau+1):(n+\tau+1)}$ only differ through the marginal laws of $Z_{\tau+1},\ldots,Z_{n+1}$ and $Z_1,\ldots,Z_{n-\tau+1}$, respectively. Consequently, by the data processing inequality, we have
\begin{align*}
\mathrm{d}_{\mathrm{TV}}\!\left(\bZ,\bZ_{(\tau+1):(n+\tau+1)}\right) 
\leq \mathrm{d}_{\mathrm{TV}}((Z_1,\ldots,Z_{n-\tau+1}),(Z_{\tau+1},\ldots,Z_{n+1})) = \Psi_{0,\tau}(\bZ),
\end{align*}
where the equality follows from the definition of the switch coefficient.

For $K\sim\mathsf{Unif}([n+\tau+1])$ drawn independently of $\bZ$, we have
\begin{align*}
&\mathrm{d}_{\mathrm{TV}}\!\left(\widetilde{\bZ}_{1:(n+1)},\mathrm{M}_{n+\tau+2-K,\tau}(\bZ)\right) \\
&= \mathrm{d}_{\mathrm{TV}}\Bigl((Z_K,\ldots,Z_{n+\tau+1},\star,\ldots,\star,Z_{\tau+1},\ldots,Z_{K-\tau-1}),(Z_1,\ldots,Z_{n+\tau+2-K},\star,\ldots,\star,Z_{n+2\tau+3-K},\ldots,Z_{n+1})\Bigr) \\
&= \mathrm{d}_{\mathrm{TV}}\Bigl((Z_K,\ldots,Z_{n+\tau+1},Z_{\tau+1},\ldots,Z_{K-\tau-1}),(Z_1,\ldots,Z_{n+\tau+2-K},Z_{n+2\tau+3-K},\ldots,Z_{n+1})\Bigr) \\
&=\mathrm{d}_{\mathrm{TV}}\Bigl(\Delta_{n+\tau+2-K}^{1}(\bZ_{(\tau+1):(n+\tau+1)}), \Delta_{n+\tau+2-K}^{0}(\bZ)\Bigr) \\
&\overset{\1}{\le} \mathrm{d}_{\mathrm{TV}}\Bigl(\Delta_{n+\tau+2-K}^{1}(\bZ), \Delta_{n+\tau+2-K}^{0}(\bZ)\Bigr) +\Psi_{0,\tau}(\bZ)\\
&\le \underline{\Psi}_\tau(\bZ) +\Psi_{0,\tau}(\bZ),
\end{align*}
where step $\1$ holds because we replace $\bZ_{(\tau+1):(n+\tau+1)}$ with $\bZ$ by paying an additive price of TV distance between the two sequences.

But by our convention~\eqref{eq:masking-convention} for masking, for $K'\sim\mathsf{Unif}(\{-\tau,\ldots,n\})$ drawn independently of everything else, we have $\mathrm{M}_{n+\tau+2-K,\tau}(\bZ)\overset{\textnormal{d}}{=}\mathrm{M}_{K',\tau}(\bZ)$. Thus
\begin{equation}
\label{inq:unf}
\rho_{\tau}\bigl(\mathrm{M}_{K',\tau}(\bZ)\bigr) \leq \underline{\Psi}_\tau(\bZ)+\Psi_{0,\tau}(\bZ),
\end{equation}
as claimed. \qed

\section{Discussion} \label{sec:disc}

Motivated by the question of whether leave-one-out predictive inference methods remain reliable under temporal dependence, we showed that the vanilla jackknife can suffer severe coverage loss even in simple time-series models exhibiting mild dependence. As a remedy, we proposed the LWO method, which modifies the leave-one-out calibration mechanism of the jackknife by removing a window of training data points after each proxy test point. Under an out-of-sample stability condition, this modification yields provable approximate coverage while retaining the data efficiency of jackknife-type methods. We note that similar principles are classical in the literature on variance estimation~\citep{kunsch1989jackknife,liu1992moving} and have recently been used in estimating (random) functionals of multinomial distributions~\citep{pananjady2024just, nakul2025estimating}. However, the use of these ideas in prediction (and specifically predictive inference) appears to be new.

Besides our methodological contribution, the main conceptual contribution of our work is a viewpoint based on approximately embedding a temporally dependent stochastic process into a cyclically exchangeable sequence. The cyclic embedding coefficient that we define is a distributional property of the stochastic process; our results connect this coefficient to classical dependence notions, including mixing and approximate Markovianity, and after a suitable masking operation, to switch coefficients. We expect that cyclic embedding will prove useful as a tool in studying other statistical problems involving temporal dependence.

Several open questions remain, and we elaborate on three of them to conclude. First, our validity guarantees for LWO rely on an out-of-sample stability condition. In the exchangeable setting, jackknife+ and CV+ avoid such stability assumptions by modifying the leave-one-out construction, at the cost of a slightly weaker coverage guarantee. It is natural to ask whether an analogous modification exists for the LWO method, which can achieve valid coverage without requiring stability assumptions. Second, it is interesting to explore if our tool of leaving a window out can be combined with cross-conformal or cross-validation style methods. Developing such versions of cross-conformal prediction, CV+, or jackknife+-after-bootstrap for dependent data would clarify how broadly the LWO principle extends beyond the jackknife in prediction problems.  Third, the choice of the window size $\tau$ remains an important practical question: our theory suggests that $\tau$ should exceed the effective dependence range, but we do not provide a fully data-driven choice of $\tau$ with valid coverage guarantees and believe this would be an interesting future direction.

\subsection*{Acknowledgments}

HJ and YX were partially supported by National Science Foundation grants DMS-2134037 and CMMI-2112533, Office of Naval Research N000142412278, and the Coca-Cola Foundation.
RFB was partially supported by the National Science Foundation via grant DMS-2023109, and
by the Office of Naval Research via grant N00014-24-1-2544.
AP was supported in part by National Science Foundation grant CCF-2107455 and a Google Research Scholar award.  

\small
\bibliography{reference}
\bibliographystyle{plainnat}
\normalsize

\appendix
\begin{center}
  \textbf{\Large{Appendix}}
\end{center}

\section{Technical lemmas and their proofs} \label{sec:app-proofs}

In this section, we collect our technical lemmas and their proofs.

\subsection{Supporting lemma for proof of Proposition~\ref{prop:cycliccoef}} \label{app:tech-lemma1}

Recall the stochastic process $\bZ^{\{k\}}$ from Eq.~\eqref{eq:Zk-def}. The following lemma bounds the total variation distance between $\bZ^{\{k\}}_{1:(n + 1)}$ and $\bZ$ for a range of $k$.

\begin{lemma} \label{lem:Zk-TV-bd}
If $2\tau + 2 \leq k \leq n - \tau+1$, then we have
\[
\mathrm d_{\mathrm{TV}}\bigl(\bZ, \bZ^{\{k\}}_{1:(n + 1)}\bigr) \leq 2 \beta(\tau) + 2\beta^*(\tau).
\]
\end{lemma}
\begin{proof}
First define subsets of indices:
\[
\begin{aligned}
A &= \{1,\ldots,n-\tau-k+2\},\\
B &= \{n-\tau-k+3,\ldots,n-k+2\},\\
C &= \{n-k+3,\ldots,n-k+\tau+2\},\\
D &= \{n-k+\tau+3,\ldots,n-k+2\tau+2\},\\
E &= \{n-k+2\tau+3,\ldots,n+1\}.
\end{aligned}
\]
To sample \(\bZ=(Z_A,Z_B,Z_C,Z_D,Z_E)\), we can equivalently follow the hierarchical sampling strategy
\[
\begin{cases}
(Z_B,Z_C,Z_D)\sim \mathbb{P}_{(Z_B,Z_C,Z_D)},\\
Z_A\sim \mathbb{P}_{Z_A\mid (Z_B,Z_C,Z_D)}\bigl(\cdot\mid Z_B,Z_C,Z_D\bigr),\\
Z_E\sim \mathbb{P}_{Z_E\mid (Z_A,Z_B,Z_C,Z_D)}\bigl(\cdot\mid Z_A,Z_B,Z_C,Z_D\bigr).
\end{cases}
\]
By definition of conditional-\(\beta\)-mixing, this distribution is at most \(2\beta^*(\tau)\) away in total variation distance from the joint distribution
\[
\begin{cases}
(Z_B,Z_C,Z_D)\sim \mathbb{P}_{(Z_B,Z_C,Z_D)},\\
Z_A\sim \mathbb{P}_{Z_A\mid Z_B}\bigl(\cdot\mid Z_B\bigr),\\
Z_E\sim \mathbb{P}_{Z_E\mid Z_D}\bigl(\cdot\mid Z_D\bigr).
\end{cases}
\]
This distribution can be equivalently written as
\[
\begin{cases}
(Z_B,Z_D)\sim \mathbb{P}_{(Z_B,Z_D)},\\
Z_A\sim \mathbb{P}_{Z_A\mid Z_B}\bigl(\cdot\mid Z_B\bigr),\\
Z_E\sim \mathbb{P}_{Z_E\mid Z_D}\bigl(\cdot\mid Z_D\bigr),\\
Z_C\sim \mathbb{P}_{Z_C\mid (Z_B,Z_D)}\bigl(\cdot\mid Z_B,Z_D\bigr).
\end{cases}
\]
Changing only the first line, and applying the definition of \(\beta\)-mixing, this is at most \(\beta(\tau)\) away in total variation distance from the joint distribution
\[
\begin{cases}
(Z_B,Z_D)\sim \mathbb{P}_{Z_B}\times \mathbb{P}_{Z_D},\\
Z_A\sim \mathbb{P}_{Z_A\mid Z_B}\bigl(\cdot\mid Z_B\bigr),\\
Z_E\sim \mathbb{P}_{Z_E\mid Z_D}\bigl(\cdot\mid Z_D\bigr),\\
Z_C\sim \mathbb{P}_{Z_C\mid (Z_B,Z_D)}\bigl(\cdot\mid Z_B,Z_D\bigr).
\end{cases}
\]
This can equivalently be written as
\[
\begin{cases}
(Z_A,Z_B,Z_D,Z_E)\sim \mathbb{P}_{(Z_A,Z_B)}\times \mathbb{P}_{(Z_D,Z_E)},\\
Z_C\sim \mathbb{P}_{Z_C\mid (Z_B,Z_D)}\bigl(\cdot\mid Z_B,Z_D\bigr).
\end{cases}
\]
By stationarity, this is equivalent to
\[
\begin{cases}
(Z_D,Z_E,Z_A,Z_B)
\sim
\mathbb{P}_{(Z_1,\ldots,Z_{k-\tau-1})}\times \mathbb{P}_{(Z_{k},\ldots,Z_{n+1})},\\
Z_C\sim \mathbb{P}_{Z_C\mid (Z_B,Z_D)}\bigl(\cdot\mid Z_B,Z_D\bigr).
\end{cases}
\]
Again applying the definition of \(\beta\)-mixing, this is at most \(\beta(\tau)\) away in total variation distance from the joint distribution
\[
\begin{cases}
(Z_D,Z_E,Z_A,Z_B)
\sim
\mathbb{P}_{(Z_1,\ldots,Z_{k-\tau-1},Z_{k},\ldots,Z_{n+1})},\\
Z_C\sim \mathbb{P}_{Z_C\mid (Z_B,Z_D)}\bigl(\cdot\mid Z_B,Z_D\bigr).
\end{cases}
\]
But, by construction, this is exactly the distribution of
\[
(Z^{\{k\}}_1,\ldots,Z^{\{k\}}_{n+1})
=
(Z^{\{k\}}_A,Z^{\{k\}}_B,Z^{\{k\}}_C,Z^{\{k\}}_D,Z^{\{k\}}_E).
\]
Taking stock, we have transformed $\bZ$ into $\bZ^{\{k\}}_{1:(n + 1)}$ up to a total variation loss of $2 \beta(\tau) + 2 \beta^*(\tau)$, as claimed.
\end{proof}

\subsection{Transferring OOS stability from test point to a random training point} \label{sec:stability-transfer}

Recall that for any $k$, we have $\bZ^{(k)}=\mathrm{M}_{k,\tau}(\bZ)$. Also recall that the LWO scores on $\bZ$ and $\bZ^{(K)}$ are given by
\begin{align*}
s_{k} &=\score(Y_{k},\Alg[\mathrm{M}_{k-1,\tau+1}(\bZ_{1:n})](\bX_{k})),\\
s_{k}^{(K)}&= \score(Y^{(K)}_{k},\Alg[\mathrm{M}_{k-1,\tau+1}(\bZ^{(K)}_{1:n})](\bX^{(K)}_k))
\end{align*}
with the convention that if $k\in\{-\tau,\ldots,0\}$, then we have $\score(Y_k, f(\bX_k)) = 0$ and so both scores defined above evaluate to zero.

We define the following notion of OOS stability evaluated at a random training point.
\begin{definition}
\label{def:stab_strong}
For a pair of scalars $(\nu, t)$, the pair $(\score, \Alg)$ is said to be OOS stable  when evaluated at a random training point with respect to leaving a random block of length $\tau+1$ out of $\bZ = (\bX_i, Y_i)_{i = 1}^{n + 1}$ if for $K \sim \mathsf{Unif}(\{-\tau,\ldots,n\})$, we have
\begin{align*}
\frac{1}{n} \sum_{k = 1}^n \mathbb{P} \Big( \Big| s_k -s_k^{(K)} \Big| > t \Big) \leq \nu.
\end{align*}
\end{definition}

The following lemma shows that if OOS stability in the sense of Definition~\ref{def:stab} holds (evaluated at the test point), then OOS stability in the sense of Definition~\ref{def:stab_strong} holds (evaluated at a random training point).

\begin{lemma}
\label{lem:strong_stab}
Suppose the pair \((\score,\Alg)\) is OOS stable with respect to \(\mathrm M_{K_0,\tau+1}(\bZ)\) with parameters \((\nu,t,\tau)\), where $K_0\sim\mathsf{Unif}(\{-\tau,\ldots,n\})$. Then this pair is OOS (random training point) with respect to \(\bZ\) in
the sense of Definition~\ref{def:stab_strong} with parameters
\[
\left(\nu+\frac{\tau}{n+\tau+1}+\frac{n+\tau+1}{n}\overline\Psi_\tau(\bZ),\,t\right).
\]
\end{lemma}

\begin{proof}
As shorthand, define the set $I:=\{-\tau,\ldots,n\}$, and let \(K\sim\mathsf{Unif}(I)\) be chosen independently of \(\bZ\). For each fixed
\(k\in[n]\), define
$
L_k:=\bigl[(K-k)\ \mathrm{mod}\ (n+\tau+1)\bigr]-\tau.
$
By the definition of the deletion operation, the event
\(|s_k-s_k^{(K)}|>t\) is a measurable function of
\(\Delta^0_{k,\tau}(\bZ)\) and the independent index \(K\). By definition of total variation distance, viewing this event as a function of \((\Delta^1_{k,\tau}(\bZ), K)\) instead can inflate the probability by at most
\(\Psi_{k,\tau}(\bZ)\). Concretely, we have
\begin{equation}
\label{eq:main-switch-step}
\begin{aligned}
\mathbb P\bigl(|s_k-s_k^{(K)}|>t\bigr)
&\le\mathbb P\Biggl(\Bigl|\score\bigl(Y_{n+1},\Alg[\mathrm M_{n-k-\tau,\tau+1}(\bZ_{1:n})](\bX_{n+1})\bigr)\\
&\quad \qquad -\score\bigl(Y_{n+1}^{(L_k)},\Alg[\mathrm M_{n-k-\tau,\tau+1}(\bZ^{(L_k)}_{1:n})](\bX_{n+1}^{(L_k)})\bigr)\Bigr|>t\Biggr)+\Psi_{k,\tau}(\bZ).
\end{aligned}
\end{equation}
Now split the first term on the RHS into two parts, and note that
$
    \mathbb{P}(L_k\in\{n+1-\tau,\ldots,n\}) = \tau/(n+\tau+1).
$
Also note that if $L_k\notin\{n+1-\tau,\ldots,n\}$, then $Y^{(L_k)}_{n+1} = Y_{n + 1}$. Therefore,
\begin{equation}
\label{eq:main-test-hit-split}
\begin{aligned}
&\mathbb P\Biggl(\Bigl|\score\bigl(Y_{n+1},\Alg[\mathrm M_{n-k-\tau,\tau+1}(\bZ_{1:n})](\bX_{n+1})\bigr)
-\score\bigl(Y_{n+1}^{(L_k)},\Alg[\mathrm M_{n-k-\tau,\tau+1}(\bZ^{(L_k)}_{1:n})](\bX_{n+1}^{(L_k)})\bigr)\Bigr|>t\Biggr)\\
&\le
\mathbb P\Biggl(\Bigl|\score\bigl(Y_{n+1},\Alg[\mathrm M_{n-k-\tau,\tau+1}(\bZ_{1:n})](\bX_{n+1})\bigr)\\
&\hspace{1cm}
-\score\bigl(Y_{n+1},\Alg[\mathrm M_{n-k-\tau,\tau+1}(\bZ^{(L_k)}_{1:n})](\bX_{n+1})\bigr)\Bigr|>t,\,L_k\notin\{n+1-\tau,\ldots,n\}\Biggr)
+\frac{\tau}{n+\tau+1}.
\end{aligned}
\end{equation}
We now claim that
\begin{equation}
\label{eq:post-switch-stability}
\begin{aligned}
&\frac1n\sum_{k=1}^{n}
\mathbb P\Biggl(\Bigl|\score\bigl(Y_{n+1},\Alg[\mathrm M_{n-k-\tau,\tau+1}(\bZ_{1:n})](\bX_{n+1})\bigr)\\
&\hspace{1.5cm}
-\score\bigl(Y_{n+1},\Alg[\mathrm M_{n-k-\tau,\tau+1}(\bZ^{(L_k)}_{1:n})](\bX_{n+1})\bigr)\Bigr|>t,\,L_k\notin \{n+1-\tau,\ldots,n\}\Biggr)\le \nu.
\end{aligned}
\end{equation}
We prove this claim shortly. Taking it as given for the moment,
averaging Ineqs.~\eqref{eq:main-switch-step} and
\eqref{eq:main-test-hit-split} over \(k=1,\ldots,n\), and then applying
claim~\eqref{eq:post-switch-stability} yields
\[
\frac1n\sum_{k=1}^{n}\mathbb P\bigl(|s_k-s_k^{(K)}|>t\bigr)
\le\nu+\frac{\tau}{n+\tau+1}+\frac1n\sum_{k=1}^{n}\Psi_{k,\tau}(\bZ).
\]
Finally, by the definition of the average switch coefficient, we have
\[
\frac1n\sum_{k=1}^{n}\Psi_{k,\tau}(\bZ)\le\frac{n+\tau+1}{n} \cdot \overline\Psi_\tau(\bZ).
\]
Therefore,
\[
\frac1n\sum_{k=1}^{n}\mathbb P\bigl(|s_k-s_k^{(K)}|>t\bigr)
\le\nu+\frac{\tau}{n+\tau+1}+\frac{n+\tau+1}{n} \cdot \overline\Psi_\tau(\bZ),
\]
as desired. It remains to prove claim~\eqref{eq:post-switch-stability}.

\paragraph{Proof of claim~\eqref{eq:post-switch-stability}.}
Recall that \(L_k\sim\mathsf{Unif}(I)\) with
$I=\{-\tau\}\cup\{1-\tau,\ldots,n-\tau\}\cup\{n+1-\tau,\ldots,n\}$. Let $U\sim\mathsf{Unif}(\{1-\tau,\ldots,n-\tau\})$ be a random variable chosen independently of \(\bZ\). 
Note that for each fixed \(k \in [n]\), we have
\begin{equation}
\label{eq:Lk-to-U-sharp}
\begin{aligned}
&\mathbb P\Biggl(\Bigl|\score\bigl(Y_{n+1},\Alg[\mathrm M_{n-k-\tau,\tau+1}(\bZ_{1:n})](\bX_{n+1})\bigr)\\
&\hspace{2.6cm}
-\score\bigl(Y_{n+1},\Alg[\mathrm M_{n-k-\tau,\tau+1}(\bZ^{(L_k)}_{1:n})](\bX_{n+1})
\bigr)\Bigr|>t,\,L_k\notin\{n+1-\tau,\ldots,n\}\Biggr)\\
&\qquad =\frac{n}{n+\tau+1}\mathbb P\Biggl(\Bigl|\score\bigl(Y_{n+1},\Alg[\mathrm M_{n-k-\tau,\tau+1}(\bZ_{1:n})](\bX_{n+1})\bigr)\\
&\hspace{4.0cm}
-\score\bigl(Y_{n+1},\Alg[\mathrm M_{n-k-\tau,\tau+1}(\bZ^{(U)}_{1:n})](\bX_{n+1})\bigr)\Bigr|>t\Biggr),
\end{aligned}
\end{equation}
where we have also used that if \(L_k=-\tau\), then $\bZ^{(L_k)} = \bZ$.
We now rewrite the rescaled average of the RHS above as a conditional probability over the masking index \(K_0 \sim \mathsf{Unif}(I)\), as follows:
\begin{equation}
\begin{aligned}
&\frac{1}{n}\sum_{k=1}^{n}\mathbb{P}\Biggl(\Bigl|\score\bigl(Y_{n+1},\Alg[\mathrm M_{n-k-\tau,\tau+1}(\bZ_{1:n})](\bX_{n+1})\bigr)
-\score\bigl(Y_{n+1},\Alg[\mathrm M_{n-k-\tau,\tau+1}(\bZ^{(U)}_{1:n})](\bX_{n+1})\bigr)\Bigr|>t\Biggr)\\
&=
\mathbb{P}\Biggl(\Bigl|\score\bigl(Y_{n+1},\Alg[\mathrm M_{K_0,\tau+1}(\bZ_{1:n})](\bX_{n+1})\bigr)
\\
&\hspace{3.2cm}-
\score\bigl(Y_{n+1},\Alg[\mathrm M_{U,\tau}(\mathrm M_{K_0,\tau+1}(\bZ_{1:n}))](\bX_{n+1})\bigr)\Bigr|>t\;\Bigg|\;K_0\in\{-\tau, \ldots, n-\tau-1\}\Biggr)\\
&=\frac{n+\tau+1}{n} \cdot \mathbb{P}\Biggl(\Bigl|\score\bigl(Y_{n+1},
\Alg[\mathrm M_{K_0,\tau+1}(\bZ_{1:n})](\bX_{n+1})\bigr)\\
&\hspace{3.2cm}-
\score\bigl(Y_{n+1},\Alg[\mathrm M_{U,\tau}(\mathrm M_{K_0,\tau+1}(\bZ_{1:n}))](\bX_{n+1})\bigr)\Bigr|>t,\,K_0\in \{-\tau,\ldots,n-\tau-1 \} \Biggr).
\end{aligned}
\end{equation}
But 
\begin{equation*}
\begin{aligned}
&\mathbb{P}\Biggl(\Bigl|\score\bigl(Y_{n+1},\Alg[\mathrm M_{K_0,\tau+1}(\bZ_{1:n})](\bX_{n+1})\bigr)\\
&\hspace{2.0cm}-
\score\bigl(Y_{n+1},\Alg[\mathrm M_{U,\tau}(\mathrm M_{K_0,\tau+1}(\bZ_{1:n}))](\bX_{n+1})\bigr)\Bigr|>t,\,K_0\in\{-\tau,\ldots,n-\tau-1\}\Biggr)\\
&\le\mathbb{P}\Biggl(\Bigl|\score\bigl((\mathrm M_{K_0,\tau+1}(\mathbf{Y}))_{n+1},\Alg[(\mathrm M_{K_0,\tau+1}(\bZ))_{1:n}]((\mathrm M_{K_0,\tau+1}(\bX))_{n+1})\bigr)\\
&\hspace{2.0cm}
-\score\bigl((\mathrm M_{K_0,\tau+1}(\mathbf{Y}))_{n+1},\Alg[\mathrm M_{U,\tau}((\mathrm M_{K_0,\tau+1}(\bZ))_{1:n})]((\mathrm M_{K_0,\tau+1}(\bX))_{n+1})\bigr)\Bigr|>t\Biggr)
\overset{\1}{\le}\nu,    
\end{aligned}
\end{equation*}
where step $\1$ follows since by assumption, we have OOS stability (in the sense of Definition~\ref{def:stab}) with respect to $\mathrm{M}_{K_0, \tau+1}(\bZ)$ with parameters $(\nu, t, \tau)$.

Putting together the pieces, we have
\begin{equation}
\label{eq:oos-after-switch}
\begin{aligned}
&\frac1n\sum_{k=1}^{n}
\mathbb P\Biggl(\Bigl|\score\bigl(Y_{n+1},\Alg[\mathrm M_{n-k-\tau,\tau+1}(\bZ_{1:n})](\bX_{n+1})\bigr)\\
&\hspace{3.2cm}
-\score\bigl(Y_{n+1},\Alg[\mathrm M_{n-k-\tau,\tau+1}(\bZ^{(U)}_{1:n})](\bX_{n+1})\bigr)\Bigr|>t\Biggr)\le\frac{n+\tau+1}{n} \cdot \nu .
\end{aligned}
\end{equation}
Averaging Ineq.~\eqref{eq:Lk-to-U-sharp} over \(k\in [n]\) and using Ineq.~\eqref{eq:oos-after-switch} yields the claimed bound.
\end{proof}

\subsection{Ancillary results on quantile stability} \label{app:ancillary}
\begin{lemma} \label{lem:quantile-subset}
If $A$ and $B$ are two (unordered) sets of real numbers such that $|A| = k$, $|B| = m$ and $A \subseteq B$, then for all $a \in (0, 1)$, we have
\begin{align*}
\Quantile_{1 - a}(A) &\geq \Quantile_{(1 - a) \cdot \frac{k}{m}}\left(B\right) \\
\Quantile_{1 - a}(B) &\geq \Quantile_{ 1 - a \cdot \frac{m}{k} }\left(A\right).
\end{align*}
\end{lemma}
\begin{proof} Both bounds follow immediately from the definition of the quantile function.\end{proof}

\begin{lemma} \label{lem:quantile-set}
Let $\bU = (U_1, \ldots, U_k)$ and $\bV = (V_1, \ldots, V_k)$ denote two vectors of real-valued random variables. If there exists a bijection $\pi:[k] \to [k]$ such that
\[
\frac{1}{k} \sum_{i = 1}^k \mathbb{P} \left\{ |U_i - V_{\pi(i)}| \geq t \right\} \leq \nu,
\]
then for any $a \in [0, 1]$, we have
\[
\mathbb{P} \left\{ \Quantile_{1 - a}(\bU) \leq \Quantile_{1 - a + \sqrt{\nu}}(\bV) + t \right\} \geq 1 - \sqrt{\nu}.
\]
\end{lemma}
\begin{proof}
We bound the probability of the complementary event. If 
\[
\Quantile_{1 - a}(\bU) > \Quantile_{1 - a + \sqrt{\nu}}(\bV) + t,
\]
then for all permutations (i.e. bijections) $\sigma: [k] \to [k]$, we must have 
\[
\#\{i: U_{i} > V_{\sigma(i)} + t\} \geq \lceil (1 - a + \sqrt{\nu}) k \rceil - (\lceil (1 - a) k \rceil - 1) \geq \sqrt{\nu} k.
\]
Indeed, if there exists a permutation such that this is not true, then it is impossible for the quantile to move by more than $t$. Thus, for the particular permutation $\pi$, we must have
\begin{align*}
\mathbb{P} \left\{ \Quantile_{1 - a}(\bU) > \Quantile_{1 - a + \sqrt{\nu}}(\bV) + t \right\} &\leq \mathbb{P}\left\{ \sum_{i = 1}^k \ind{U_{i} > V_{\pi(i)} + t} \geq \sqrt{\nu} k \right\} \\
&\overset{\1}{\leq} \frac{\sum_{i = 1}^k \mathbb{P} \left\{|U_{i} - V_{\pi(i)}| > t \right\}}{\sqrt{\nu}k} \overset{\2}{\leq} \sqrt{\nu}.
\end{align*}
Step $\1$ is due to Markov's inequality, and step $\2$ follows from the assumption on $\pi$. 
\end{proof}

\section{Additional experiments} \label{sec:add_expt}

In this section, we include a broader class of simulation and real data experiments, complementing our empirical investigations from Section~\ref{sec:experiments}.

\subsection{Simulation of Split CP, vanilla Jackknife and LWO with 1D response}
\label{App:add_exp}
Our aim in this section is to construct an extreme example (with one-dimensional responses) that highlights the worst-case behavior we might see for jackknife in the time series setting. In particular, we will construct a data distribution and a regression algorithm such that stability does hold but the coverage of jackknife is severely impacted if we do not follow the LWO modification.

Concretely,
we consider the following epoch-based Markov chain. At each time \(t\ge 1\), let
    \[
    (T_t,K_t,X_t,Y_t)\in \mathbb{N}\times \mathbb{N}\times \mathbb{R}\times \mathbb{R},
    \]
    where \(K_t\) is latent while \((X_t,Y_t)\) are observed. The chain evolves at time \(t+1\) as follows:
    \begin{itemize}
        \item If \(T_t=K_t\), a new epoch starts:
        \[
        T_{t+1}=1,
        \qquad
        K_{t+1}\sim \mathrm{Geom}(\rho),
        \qquad
        X_{t+1}\sim \mathcal{N}(0,1),
        \qquad
        Y_{t+1}\sim \mathcal{N}(K_{t+1},1).
        \]
        \item If \(T_t<K_t\), the current epoch continues:
        \[
        T_{t+1}=T_t+1,
        \qquad
        K_{t+1}=K_t,
        \qquad
        X_{t+1}=X_t,
        \qquad
        Y_{t+1}\sim \mathcal{N}(K_t,1).
        \]
\end{itemize}
Thus \(X_t\) is constant within each epoch, while \(Y_t\) has mean equal to the latent epoch length \(K_t\). Since new epoch values of \(X_t\) are drawn from a continuous distribution, two time points have the same raw covariate value almost surely only when they belong to the same epoch. We set \(\rho=0.05\) and \(n=1000\).

The predictor does not use the raw covariate \(X_t\) directly. Instead, the fitted algorithm is a composition of a data-dependent feature map and a chosen base predictor, both constructed from the training data. Specifically, given a training set $\mathcal{D}$ of covariate-response pairs,
we first learn the feature map
\[
\widehat\phi_{\mathcal D}(x)
:=
\sum_{(X_i,Y_i)\in\mathcal D}\mathbf 1\{X_i=x\}.
\]
This map counts how many training observations share the same raw covariate value as \(x\). 
We then transform each training point \((X_i,Y_i)\in\mathcal D\) into
\[
\bigl(\widehat\phi_{\mathcal D}(X_i),Y_i\bigr),
\]
and fit a base regression model \(\widehat g_{\mathcal D}\) on these transformed data points. The final predictor $\widehat{f}$ trained on \(\mathcal D\) is given by $\widehat{g}_{\mathcal{D}} \circ \widehat{\phi}_{\mathcal{D}}$. In particular, its evaluation at any point $x$ is given by 
\[
\widehat f_{\mathcal D}(x)
=
\widehat g_{\mathcal D}\!\left(\widehat\phi_{\mathcal D}(x)\right).
\]
We use no additional lagged history, i.e., \(L=0\).
Note that the same learned feature map \(\widehat\phi_{\mathcal D}\) is used to featurize any calibration or test point evaluated by this fitted model. Since both \(\widehat\phi_{\mathcal D}\) and \(\widehat g_{\mathcal D}\) depend on the training dataset $\mathcal{D}$, they are computed differently in the different predictive inference methods. For the jackknife \(\mathcal D\) can be any set of $n - 1$ data points. In split conformal, \(\mathcal D\) is the proper training split. In an LWO fit, \(\mathcal D\) is such that it excludes the local window after the proxy test point. 

Note that for this stochastic process, the epoch lengths are geometric with parameter \(\rho\), the probability that dependence persists across a gap of length \(\tau\) decays as $(1-\rho)^\tau$. Consequently, the process is geometrically \(\beta\)-mixing, with mixing coefficients
$
\beta(\tau)= O( (1-\rho)^\tau).
$

\begin{figure}
    \centering
    \includegraphics[width=0.8\linewidth]{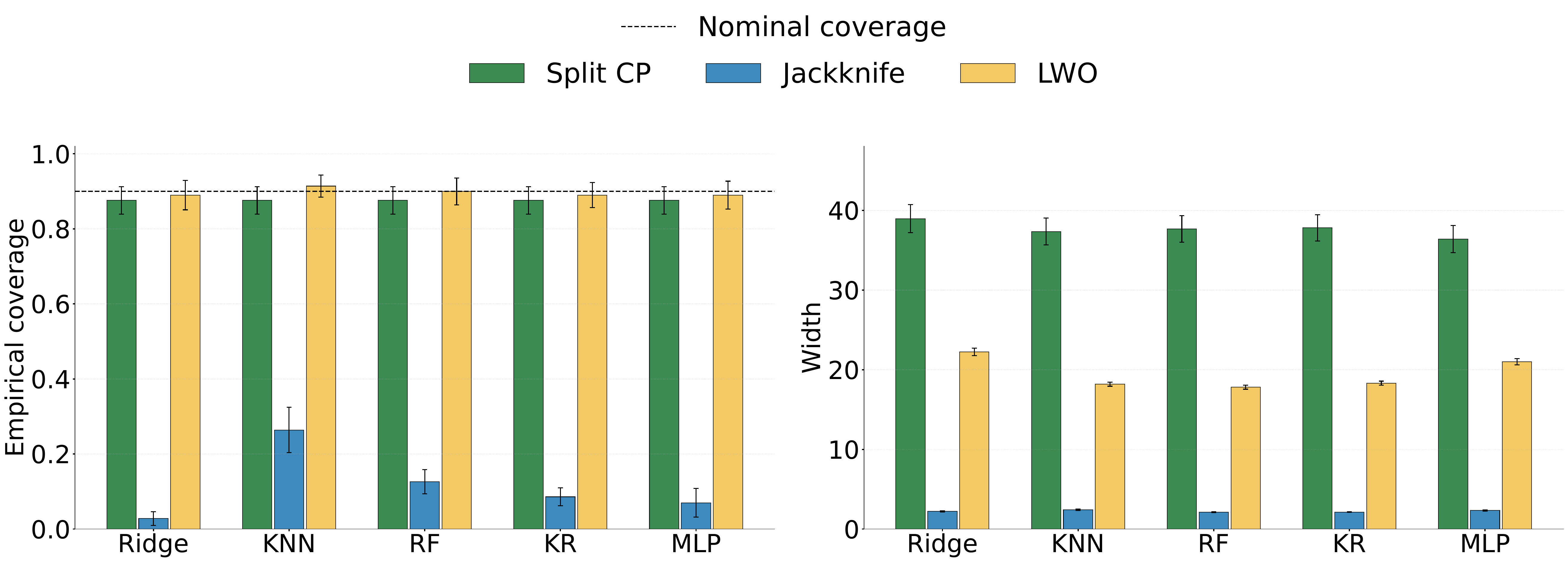}
    \caption{Empirical performance on the sticky Markov-chain process across five base predictors. 
(Left) Empirical coverage of Split CP, Jackknife, and LWO, with the dashed horizontal line marking the nominal \(90\%\) coverage level. 
(Right) Average prediction radius for the same methods and predictors. 
Bars show means across \(500\) repeated trials, and error bars indicate \(\pm 1\) standard error.}
    \label{fig:markov_results}
\end{figure}

\paragraph{Evaluation}
Figure~\ref{fig:markov_results} shows that the sticky Markov chain is particularly challenging for leave-one-out methods. The vanilla jackknife produces much narrower intervals, but these intervals are overly optimistic and lead to severe undercoverage for most predictors. The failure mechanism comes from the feature construction. Recall that \(\widehat\phi_{\mathcal D}(x)\) counts how many training observations have the same covariate value as \(x\), which, in this process, is essentially the number of observed points from the same latent epoch. For interior proxy points used by the jackknife, the surrounding training data typically contain almost the entire epoch, both before and after the proxy point. Thus \(\widehat\phi_{\mathcal D}(X_k)\) is a good proxy for the latent epoch length \(K_k\), making these jackknife proxy points relatively easy to predict and producing small calibration scores.

The final test point is different. At the end of the observed sequence, the current epoch is right-censored: we only observe the part of the epoch that has occurred so far, and do not observe future points from the same epoch. Consequently, \(\widehat\phi_{\mathcal D}(X_{n+1})\) tends to underestimate the true epoch length. Since \(Y_{n+1}\) has mean tied to the latent epoch length, this makes the fitted predictor systematically less accurate at the actual test point than at the interior jackknife proxy points. This mismatch causes the jackknife calibration scores to be too small and leads to undercoverage.

LWO mitigates this asymmetry by removing a local window after each proxy point. This censors the proxy point's epoch in the same way the final test epoch is censored, preventing the fitted model from using nearby future observations that would not be available at test time. As a result, the LWO calibration scores reflect the difficulty of predicting the final test point. 

\subsection{Comparison with other predictive inference methods}

Our aim in this section is to compare LWO against a broader range of methods, focusing in particular on (1) cross-validation methods with various choices of the number of folds, and (2) variants of conformal prediction that have been proposed for handling dependence in the time series setting.

Concretely, we compare LWO, split CP and Jackknife with the following additional predictive inference procedures with the same base predictors used in Section~\ref{sec:experiments}. For vector-valued responses, we use the \(\ell_2\)-residual score and for scalar responses we use the absolute residual score.

\begin{itemize}
    \item \textbf{Split Localized CP (SLCP):} a localized variant of split CP that estimates local residual quantiles through kernel localization~\citep{han2022splitlocalized}. 
    
    \item \textbf{Block-Randomization CP (Block CP):} a split version of block-randomization conformal method from the block permutation framework of~\citet{chernozhukov2018exact}. 
    We use a block size of $10$ for our simulation experiments and a block size of $3$ for our real data experiments.

    \item \textbf{Multi-Split CP (Multi-split CP):} an aggregation of split conformal prediction sets over many random train/calibration splits~\citep{solari2022multisplit}. A candidate response is included if it belongs to a sufficiently large fraction of the split conformal sets.

    \item \textbf{Kernel-based Optimally Weighted CP Intervals (KOWCPI):} a time-series conformal method based on reweighted Nadaraya--Watson quantile estimation of future nonconformity scores~\citep{lee2025kernel}. 

    \item \textbf{Distributional CP (DCP):} a method that estimates the conditional distribution of the response and conformalizes probability-integral-transform ranks~\citep{chernozhukov2021distributional}. In our implementation, DCP fits a Quantile Random Forest estimator~\citep{meinshausen2006quantile} with \(200\) trees and minimum leaf size \(2\). Note that this method does not depend on the base predictor. 

    \item \textbf{Conformalized Quantile Regression (CQR):} a method that first estimates lower and upper conditional quantiles and then conformalizes the resulting interval~\citep{romano2019conformalized}. In our implementation, CQR fits a Quantile Random Forest estimator with \(200\) trees and minimum leaf size \(2\). Like DCP, it does not depend on the base predictor.

    \item \textbf{\(K\)-fold Cross-Validation CP (CV):} the \(K\)-fold analogue of the vanilla jackknife, using out-of-fold residuals to calibrate a prediction set centered at the full-sample predictor~\citep{barber2021predictive}. We consider \(K\in\{5,30,100\}\): smaller \(K\) leaves out larger folds and therefore behaves more like split CP, while larger \(K\) leaves out smaller folds; we obtain the vanilla jackknife if $K = n-1$.
\end{itemize}

\subsubsection{Multidimensional MA($q$) process}
As in the setup of Section~\ref{sec:motivating}, we generate our covariates as a sequence of random vectors \(X_t\in\mathbb R^d\). In this experiment, however, we use a moving-average process with memory $q$. Specifically, in dimension \(d\), let
\[
\omega_t \stackrel{\mathrm{i.i.d.}}{\sim} \mathcal N(0,I_d),
\]
and define
\[
X_t
=
\sum_{j=0}^{q}\omega_{t-j}.
\]
The response at time \(t\) is given by
\[
Y_t=X_{t+1}.
\]
We choose \(d=50\), \(q=10\) and consider the same memoryless setting with \(\bX_t=X_t\). Thus, the prediction problem is to use the current observation \(X_t\) to forecast the next vector \(X_{t+1}\). As in the previous experiment, we set the sequence length \(n=200\).
Compared with the MA($1$) process in Section~\ref{sec:motivating}, this MA($10$) construction has a longer local dependence window. 

\begin{figure}
    \centering
    \includegraphics[width=\linewidth]{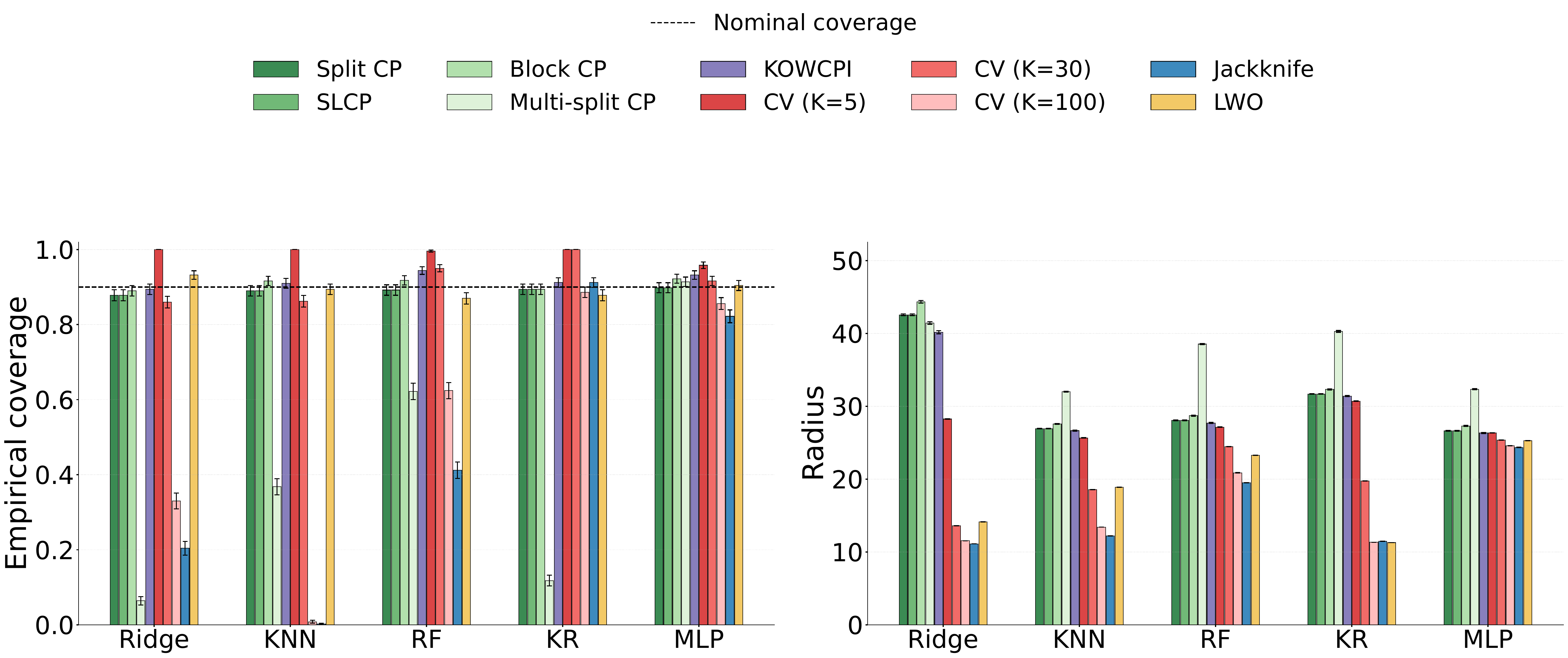}
\caption{Empirical performance on the multidimensional MA(10) process across five base predictors. 
(Left) Empirical coverage of different CP methods with the dashed horizontal line marking the nominal \(90\%\) coverage level. 
(Right) Average prediction region radius for the same methods and predictors. 
Bars show means across $500$ repeated trials, and error bars indicate \(\pm 1\) standard error.}
    \label{fig:ma10_add_results}
\end{figure}

Figure~\ref{fig:ma10_add_results} compares the methods on the multidimensional MA($10$) process. The main pattern is that methods in the same family tend to behave similarly.
Most of the methods based on sample-splitting---in particular Split CP, SLCP, and Block CP---are generally robust in coverage, with Block CP typically being somewhat more conservative because it calibrates through a block-randomization correction. On the other hand, Multi-Split CP exhibits loss of coverage for most predictors.

The CV family illustrates the interpolation between split conformal and leave-one-out methods. For small \(K\), each fold removes a large block of the training data, so the out-of-fold residuals are relatively large and the resulting intervals can be conservative. As \(K\) increases, the folds become smaller, the intervals shrink, and the method behaves increasingly like the jackknife. In the MA($10$) process, this transition is visible: the large-\(K\) CV variant inherits much of the undercoverage of the vanilla jackknife, because leaving out only a small fold does not sufficiently break the local dependence.

The KOWCPI method is often competitive because it estimates future score quantiles from residual histories, but its performance is not as uniform across predictors. LWO maintains near-nominal coverage while remaining substantially more efficient than the more conservative valid baselines in most cases.

\subsubsection{Real data}

\begin{figure}[ht!]
\centering

\includegraphics[width=\linewidth]{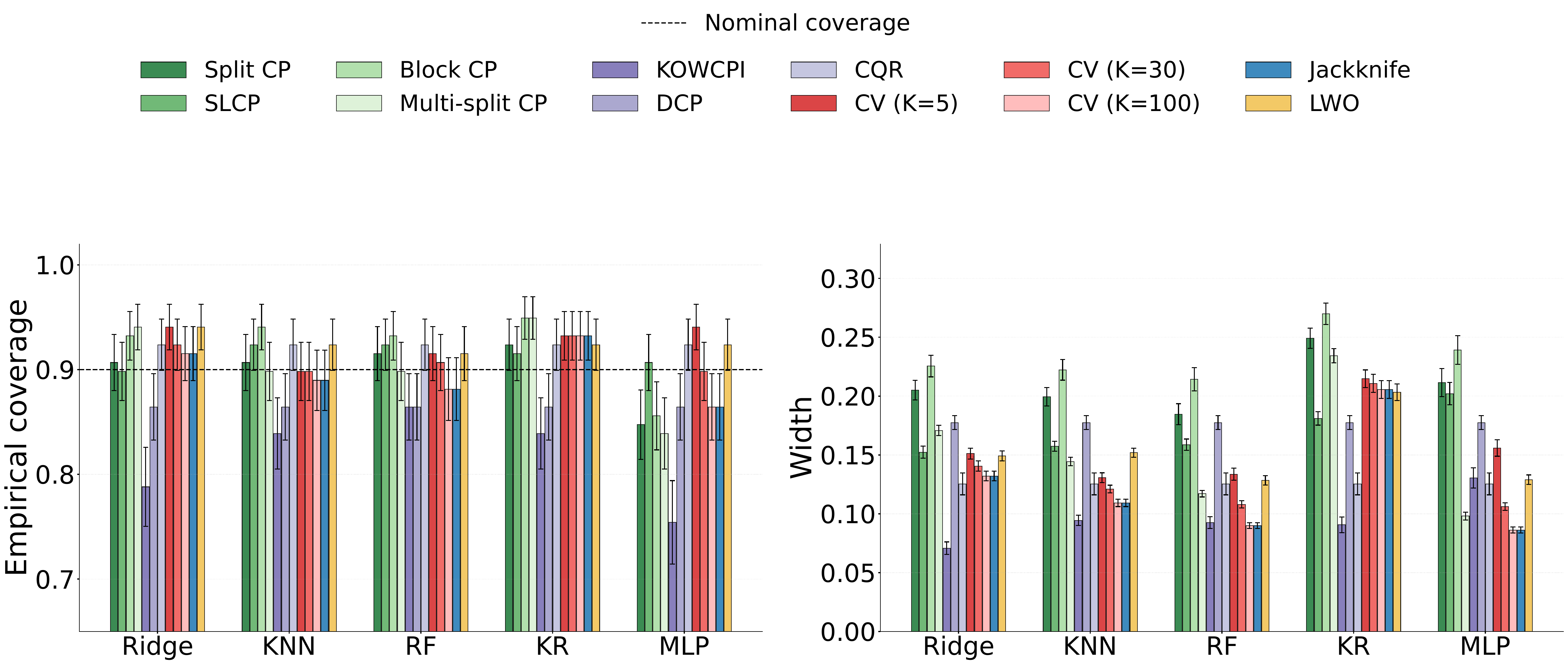}

\vspace{0.8em}

\includegraphics[width=\linewidth]{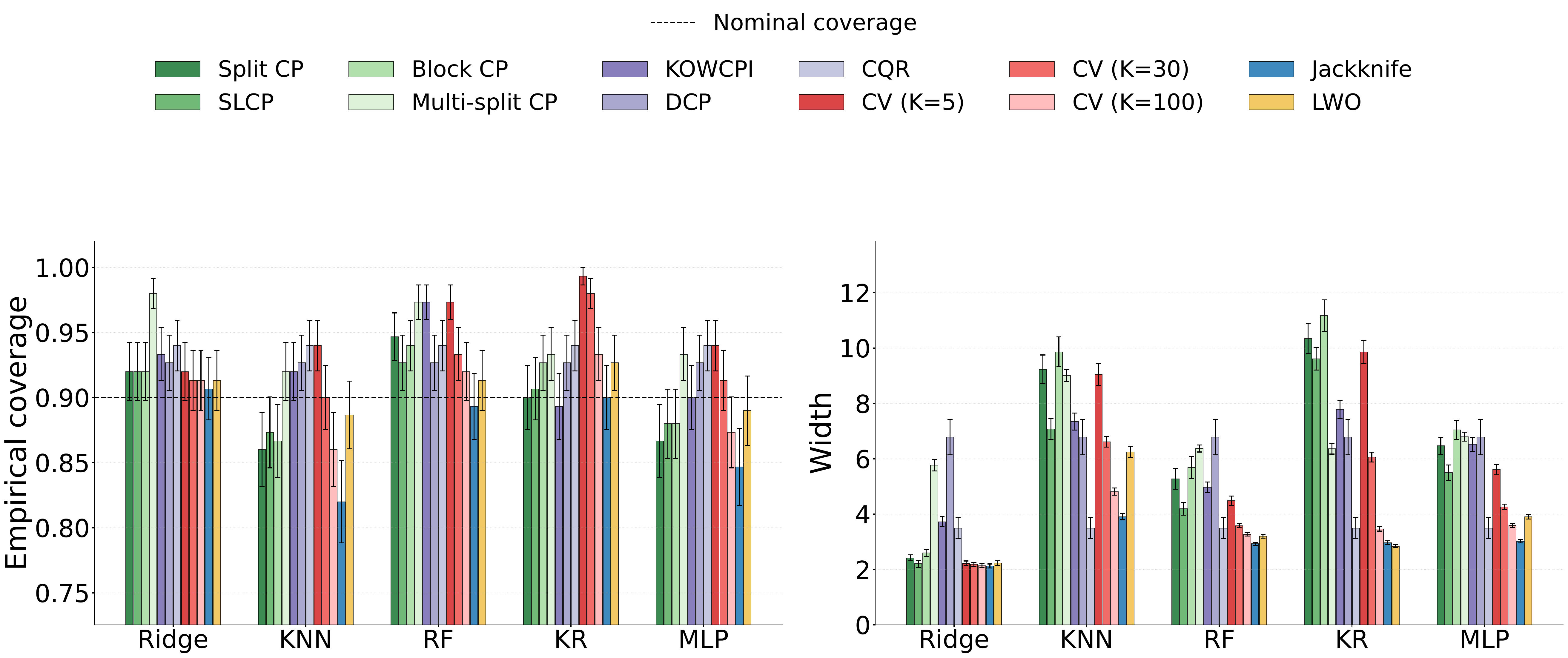}

\caption{Empirical performance on two real data benchmarks. (Top) Traffic dataset. (Bottom) Solar Energy dataset. In each row, the left panel reports empirical coverage, and the right panel reports average prediction-set size across base predictors. Bars show means across repeated trials, and error bars indicate \(\pm 1\) standard error. The dashed horizontal line marks nominal \(90\%\) coverage.}
\label{fig:real_data_add_results}
\end{figure}

Figure~\ref{fig:real_data_add_results} reports the same comparison on the Traffic and Solar Energy datasets. The real-data results again show a clear separation between method families. Methods in the split family are usually robust in coverage, but they differ noticeably in efficiency. 
SLCP often reduces the width relative to Split CP by using local residual information, while Block CP tends to be more conservative because it calibrates through a block-randomization correction. Multi-split CP can also be conservative because  each individual split set is made more conservative, and the final set keeps only points appearing in enough splits.

Quantile regression based methods (in purple) behave differently. KOWCPI can produce short intervals because it adapts to recent nonconformity-score history, but this adaptivity can also lead to undercoverage when the residual-history model is too local or poorly matched to the data. DCP and CQR fit their own conditional distribution estimators. DCP is generally conservative on the real datasets, while CQR is also stable in coverage but can be inefficient when the fitted quantile bands are wide. These methods therefore serve as adaptive distribution-estimation baselines, but their efficiency depends heavily on the quality and stability of the underlying distribution or quantile estimates.

The CV family again shows the expected interpolation. Smaller \(K\) behaves more like split CP because each out-of-fold fit leaves out a large fold, producing larger residuals and more conservative intervals. Larger \(K\) approaches the jackknife and therefore inherits more leave-one-out behavior: intervals shrink, but the method can become less robust when local dependence makes the leave-out residuals optimistic. This trend is especially visible when comparing the \(K=5\), \(K=30\), and \(K=100\) variants across predictors.

Finally, and as in the previous set of experiments, LWO maintains a good tradeoff between coverage and efficiency even when compared with this broader class of predictive inference methods.
\end{document}